\documentclass{article}

\usepackage{iclr2027_conference,times}
\newif\ifarxiv
\arxivtrue
\ifarxiv\iclrfinalcopy\fi
\usepackage[T1]{fontenc}
\usepackage{amsmath,amssymb,amsthm}
\usepackage{booktabs}
\usepackage{longtable}
\usepackage{graphicx}
\usepackage{microtype}
\usepackage{xcolor}
\usepackage{array}
\usepackage{url}
\usepackage[colorlinks=true,linkcolor=blue!55!black,
            citecolor=blue!55!black,urlcolor=blue!55!black]{hyperref}

\newtheorem{lemma}{Lemma}

\newcommand{\pp}{\,\mathrm{pp}}
\newcommand{\task}[1]{\textsc{\MakeLowercase{#1}}}
\newcommand{\papertitle}{A Helps B While B Hurts A: Directed Transfer in Instruction-Tuning Mixtures}
\ifarxiv\hypersetup{pdftitle={\papertitle},pdfauthor={Nima H. Siboni, Vahid Rostami}}\fi

\title{\papertitle}
\ifarxiv
\author{Nima H. Siboni$^{1,2}$ \hspace{2em} Vahid Rostami$^{2,*}$ \\[5pt]
{\small $^{1}$Juna.ai, Kastanienallee 32, 10435 Berlin, Germany} \\
{\small $^{2}$Computational Systems Neuroscience, Institute of Zoology, University of Cologne, Germany} \\
{\fontsize{8}{10}\selectfont $^{*}$Corresponding author: \texttt{vrostami [at] uni-koeln [dot] de}}}
\else
\author{Anonymous authors}
\fi

\begin{document}
\maketitle
\ifarxiv\lhead{Preprint}\fi

\begin{abstract}
    Adapting a language model to a specialized corpus means choosing
    which instruction-tuning tasks to train on under a fixed budget, and 
    testing one choice costs a fine-tuning run. Common heuristics add more source tasks or pick sources similar to the target. 
    The first assumes transfer is never negative; the second, that it is symmetric. 
    We show that both assumptions fail: task $A$ can help task $B$ while $B$ hurts $A$, so
    helpfulness is a signed property of \emph{ordered} source--target pairs. We introduce the \emph{transfer
    map}, a signed estimate of how much each source helps or hurts each
    held-out target. We fit the map in hundreds of fine-tuning runs on Qwen3 and Mistral models from 0.6B to 32B parameters, 
    with all sources drawn from one corpus and no training examples from the target. 
    The map predicts a held-out target's accuracy on unseen mixtures: recorded before those runs, its predictions
    have less than half the error of a mixture-agnostic baseline.
    The map is specific to its target and corpus but transfers across
    model scale: a mixture selected in advance at one
    size beats training on all source tasks at every other size we tested.
    Transfer is thus a property of the data. The map selects the tasks
    that help and drops the one that interferes: accuracy on the reasoning
    targets (causal explanation, multi-hop questions and methodological
    critique) rises by up to 14 percentage points over training on all
    source tasks.
\end{abstract}

\section{Introduction}
\label{sec:intro}

Domain adaptation of language models increasingly relies on synthetic supervision
\citep{nayak2024bonito,yang2025synthetic,li2025scilitllm}. A generation pipeline converts each passage
of a domain corpus into instruction--response pairs of several task types, such as
definitions or multi-hop questions, and the model trains on all of them together
\citep{cheng2024instructpt,cheng2024adaptllm}.
Under a fixed example budget, every task type added displaces examples of the others.
The practitioner must therefore choose which task types to train on (the \emph{sources})
to serve the task whose accuracy matters (the \emph{target}). Testing a candidate mixture
takes a fine-tuning run, and seven source types already give $127$ mixtures.

Practitioners make this choice with one of two heuristics, each resting on
an assumption about transfer. The \emph{count heuristic}
trains on every task type the pipeline produces, following evidence that
more tasks improve held-out performance
\citep{wei2022flan,wang2022supernaturalinstructions,chung2024scaling}. It
assumes that no task type hurts the target, yet negative transfer between
fine-tuning tasks is common
\citep{aribandi2022ext5,jang2023exploring,dong2024abilities}, and supervised
fine-tuning on math reasoning often erodes general capabilities
\citep{huan2026transferability}. Multi-domain
reinforcement learning of reasoning models shows a similar interference
\citep{yang2026perturbation,mcpo2026,onestep2026}. The
\emph{similarity heuristic} trains on the sources (or examples) most similar
to the target, with similarity measured on instructions
\citep{lee2024instructionmatters} or sparse-autoencoder features
\citep{ma2025monosemantic}. A similarity score is the same in both
directions, so this heuristic assumes that transfer is symmetric.

Both assumptions fail between task types generated from one corpus. Replacing
an average source with methodological critiques (\task{ERR}) lowers multi-hop
(\task{MHOP}) accuracy by $9$ to $12$ percentage points (pp) at every Qwen3 size
and on Mistral-24B (Section~\ref{sec:scope}). Transfer also has a direction: on
Qwen3-32B, replacing an average source with definitions (\task{DEF}) raises
\task{MHOP} accuracy by $5.3\pp$, while adding \task{MHOP} to a mixture lowers
\task{DEF} accuracy by $1.9\pp$ in absolute terms (Section~\ref{sec:asym}).
Transfer can be negative, and where its directions differ, no symmetric score
can represent it.

\begin{figure}[t]
  \centering
  \includegraphics[width=\textwidth]{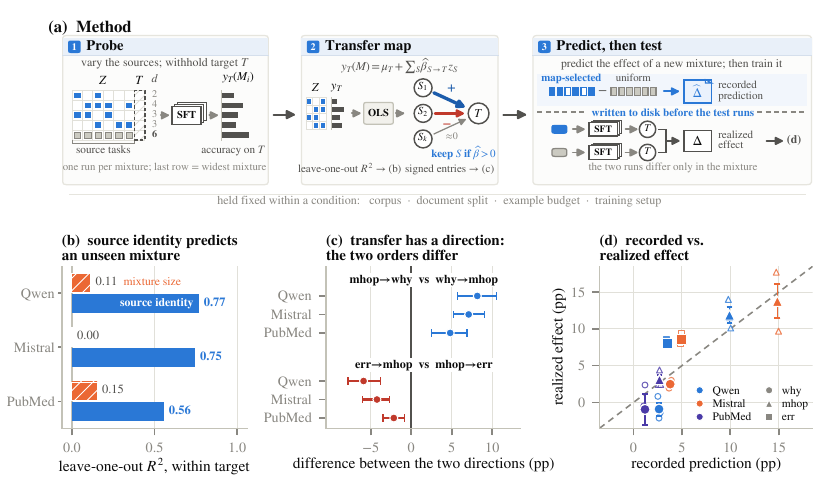}
  \caption{\textbf{Probe, map, predict.} \textbf{(a)} Balanced mixtures over
  one corpus, each withholding a target, give a signed, directed map whose
  predictions for unseen mixtures are recorded before the runs that test them.
  \textbf{(b)} Source identity predicts an unseen mixture and mixture size
  predicts little (within-target $R^2$, Table~\ref{tab:cv}).
  \textbf{(c)} The two directions of a task pair differ. \textbf{(d)}
  Recorded predictions against realized effects in the eight cells of
  Section~\ref{sec:predict}: filled markers are three-seed means ($\pm1$
  standard error), open markers single seeds.}
  \label{fig:overview}
\end{figure}

We ask five questions, and the contributions below answer them in order.
\begin{enumerate}
\itemsep0pt\parskip0pt\topsep2pt
\item \textbf{Does transfer between task types keep a sign and a direction when
the corpus is held fixed?} Sign reversals appear between separate datasets
\citep{kim2023taskweb,krishna2025latent}, where task type and content vary
together and the dataset matters more than the task category
\citep{dataprophet2026}; we measure inside fixed-budget mixtures
(Section~\ref{sec:asym}).
\item \textbf{When accuracy rises with the number of tasks, is it the number
that helps, or the average added task?} The evidence behind the count heuristic
varies both together.
\item \textbf{How much of a target's accuracy is at stake in the choice of
sources?} Selected subsets can beat the full set
\citep{kim2023taskweb,lee2024instructionmatters,renduchintala2024smart}; we ask
whether one source can cancel the benefit of the rest.
\item \textbf{Can a few runs measure transfer well enough to predict, before
training, an unseen mixture's accuracy on a target that has no training
examples?} Regressions over sampled mixtures already predict unseen ones for
pretraining domains \citep{liu2025regmix,ye2024mixinglaws} or for a target
trained alongside its sources (\citealp{li2023identification};
\citealp{li2025sftmix}).
\item \textbf{Does the measurement belong to the data or to the model?}
Data choices made with one model carry to other models in some studies
\citep{xie2023doremi,liu2025regmix,dymkiewicz2025donors} but only partly in
others \citep{ado2024,spice2026}.
\end{enumerate}

To answer them, we measure transfer instead of assuming it. The \emph{transfer
map} gives the signed effect of each source on each target
(Figure~\ref{fig:overview}a). We estimate it from a \emph{probe}, a set of
mixtures that each split a fixed budget equally among their sources and
withhold the target, as the held-out-task evaluations behind the count
heuristic do \citep{wei2022flan,wang2022supernaturalinstructions}, so the
target's accuracy measures transfer from the sources. The estimator is
deliberately simple, a linear datamodel over task types
\citep{ilyas2022datamodels,li2023identification}; what is new is the
controlled setting in which we fit it and what it reveals. We generate all eight task types from one neuroscience corpus.
Three of them, causal explanations (\task{WHY}), multi-hop questions (\task{MHOP}) and methodological critiques (\task{ERR}), are our reasoning targets. Without fine-tuning, our
two largest models score below $22\%$ on them, against about $79\%$ on
single-answer multiple choice, so they are the tasks adaptation is for. A
corpus of cancer-immunotherapy abstracts serves as a corpus-shift control. The
study spans $751$ fine-tuning runs on five Qwen3 models from $0.6$B to $32$B
parameters and on Mistral-24B, a model from a second family.\looseness=-1

\paragraph{Contributions.}
\begin{itemize}
\itemsep0pt\parskip0pt\topsep2pt
\item \textbf{Transfer is signed and directed.} Task $A$ can help task $B$ while
$B$ hurts $A$, on both model families, between task types from one corpus and
inside fixed-budget mixtures: the two directions differ in $14$ and $13$ of $28$
pairs on the two families and have opposite signs, each direction's interval excluding
zero, in $3$ and $4$ (Section~\ref{sec:asym}).
\item \textbf{Which sources matter, not how many.} When every source may have its
own effect, no fit can separate mixture size from the average source effect
(Lemma~\ref{lem:alias}). Measured slopes on mixture size track the mean
source effect ($r=0.90$ over $12$ cells, Table~\ref{tab:dose}) and are negative or flat for definitions.
Evidence that more tasks help therefore cannot justify adding every task
(Section~\ref{sec:sizes}).
\item \textbf{One interfering task type can cancel the benefit of all the others.}
On Qwen3-32B, training on all other task types in equal parts (the \emph{uniform
mixture}) barely beats the untuned model on multi-hop questions ($20\%$ to $22\%$) and
critiques ($7\%$ to $8\%$). The map's mixture drops the interfering type and two others, doubles
critique accuracy (to $16\%$) and raises multi-hop accuracy to $34\%$
(Section~\ref{sec:predict}, Figure~\ref{fig:absolute}).
\item \textbf{The map predicts and selects mixtures before training.} In two
tests, of $8$ and $12$ cells, its recorded predictions have less than half the
error of a mixture-agnostic baseline. As a selector, it beats the similarity
selector at every mixture size we tried (single-run comparisons), by $3$ to $16\pp$ on the neuroscience
multi-hop and critique targets, because similarity does not measure a source's
effect. A map refit on $24$ randomly drawn probe runs recovers $96\%$ of the full probe's out-of-fold gain on
Qwen3-32B (Sections~\ref{sec:predict} to~\ref{sec:scope}).
\item \textbf{Transfer is a property of the data, shared across models.} Maps
fitted on two model families agree ($r=0.88$), and so do maps fitted at five
sizes over a $53$-fold range ($r=0.76$ to $0.92$). A mixture selected from the
$32$B map wins in $11$ of $12$ cells at the smaller sizes. The map differs on
a second corpus with much smaller training pools ($r<0.55$; Section~\ref{sec:scope}).
\end{itemize}
\section{Related Work}
\label{sec:related}

\paragraph{Transfer between tasks.}
Transfer has been measured between vision tasks
\citep{zamir2018taskonomy,standley2020tasks,fifty2021tag} and between language or
multimodal datasets
\citep{vu2020transfer,pruksachatkun2020intermediate,poth2021pretrain,%
kim2023taskweb,krishna2025latent,dataprophet2026}. Transfer
is often asymmetric \citep{dataprophet2026,dymkiewicz2025donors}, and the two
directions of a task pair can even differ in sign \citep{kim2023taskweb,krishna2025latent}.
Surface similarity predicts transfer poorly \citep{krishna2025latent}. These studies train
one source at a time or several tasks jointly; none splits a fixed budget
among sources. In the language and multimodal studies, each source is a separate
dataset, so task type is confounded with data source. We generate all task types
from one corpus and measure directed
transfer in fixed-budget mixtures that withhold the target
(Section~\ref{sec:asym}).

\paragraph{Number of tasks.}
Training on more tasks often helps
\citep{wei2022flan,wang2022supernaturalinstructions,chung2024scaling,aribandi2022ext5},
but a selected subset can outperform the full set
\citep{renduchintala2024smart,kim2023taskweb,lee2024instructionmatters}. A
mixture's size is the sum of its inclusion indicators, so a gain from more
tasks shows only that the average added task helps (Lemma~\ref{lem:alias}).
\citet{jung2026mixture} are closest to our design: they train on every subset
of five instruction tasks at four budgets, evaluate on the same five tasks,
and find that the best number of tasks grows with the budget. This trend is
equally consistent with an average source effect that changes with the
budget.

\paragraph{Target-directed selection.}
Similarity selectors compare each candidate with the target, using
instruction embeddings \citep{lee2024instructionmatters} or sparse-autoencoder
features \citep{ma2025monosemantic}. Such a score is symmetric and measures
resemblance, not benefit, so it cannot flag a similar source as harmful
(Section~\ref{sec:buys}). LESS \citep{xia2024less} estimates benefit
rather than resemblance: it scores each example by a signed, first-order
gradient approximation of its influence on the target loss. Datamodels predict a model's output from which training
examples it saw \citep{ilyas2022datamodels} and can select pretraining data
for target tasks \citep{engstrom2024dsdm}. The transfer map is a datamodel
over task types, fitted directly to fine-tuning runs that withhold the target.

\paragraph{Mixture optimization.}
Mixture methods tune continuous domain proportions for pretraining
\citep{xie2023doremi,ye2024mixinglaws,liu2025regmix} or fine-tuning
\citep{li2025sftmix}; some fit validation loss as a function of the
proportions to predict unseen mixtures. SMART \citep{renduchintala2024smart}
and TaskPGM \citep{taskpgm2025} build one instruction-tuning mixture over a pool
of training tasks from symmetric task similarities. ADAPT
\citep{kadasi2025adapt} learns task proportions under a token budget from
meta-gradients of the training tasks' own validation losses. We instead choose which
sources to include for one withheld target. Skill-It \citep{chen2023skillit}
comes closest to a transfer map: it builds a directed skill graph from
single-skill or pairwise runs, including edges to held-out skills, and uses
it to reweight the mixture during training. Data chosen with a smaller model
can train a larger one
\citep{xie2023doremi,liu2025regmix,chen2023skillit,xia2024less,shihab2026proxymix},
and data effects correlate across model scale \citep{khaddaj2025small}; we show
that at the level of task types the map's decisive entries are shared across
model families and sizes (Section~\ref{sec:scope}).

\paragraph{Generated domain supervision.}
Domain-adaptation pipelines generate several task types and train on all of
them \citep{cheng2024instructpt,cheng2024adaptllm,li2025scilitllm}. This
default is the count heuristic and the uniform mixture that we measure against. Bonito \citep{nayak2024bonito} instead
generates data of the target's own task type, which it assumes is known; we
ask which other types to add. Fine-tuning can erode capabilities that its
data do not cover \citep{luo2025forgetting}, so dropping types may cost
ability outside the target; Section~\ref{sec:scope} reports this trade-off.
\section{Method}
\label{sec:setup}

\paragraph{Corpora and conditions.}
The primary corpus is $110$ papers on motor control and its disorders from one
publicly funded research consortium%
\ifarxiv\footnote{The papers are publications of the Collaborative Research Centre
(CRC) 1451, ``Key Mechanisms of Motor Control in Health and Disease,''
funded by the Deutsche Forschungsgemeinschaft (DFG, German Research
Foundation); see \url{https://www.crc1451.uni-koeln.de/}.}\fi,
split by publication date into the $88$
oldest papers for training and the $22$ newest for evaluation ($14{,}618$
evaluation items).
Its narrowness is deliberate: it keeps the content behind the eight task
types similar, so the types differ mainly in what they ask. A
\emph{condition} pairs a model with a corpus, and we use three. The first is
Qwen3-32B (the post-trained release, run with thinking off) on this corpus. The
second, Mistral-Small-3.2-24B-Instruct-2506 (Mistral-24B) on the same corpus,
controls for model identity: it shares corpus, tasks, budget, split and scoring
rule. The third, Qwen3-32B on $951$ PubMed cancer-immunotherapy abstracts ($761$
for training, $190$ for evaluation, $1{,}464$ items), controls for corpus shift.
Its corpus, document granularity, training-pool size (generated training examples per task type) and evaluation size change
together, so we treat its effects as less well resolved throughout.
Section~\ref{sec:scope} adds four smaller Qwen3 sizes, $0.6$B to $8$B, each
with a reduced probe. A \emph{run} is one fine-tuning of one mixture at one
seed, scored on the full evaluation split; a \emph{cell} pairs a condition
with a withheld target, written condition/target (e.g., Mistral/\task{MHOP}).\looseness=-1

\paragraph{Tasks.}
A structured generator using GPT-5-mini renders corpus passages into eight task types:
cloze (\task{CLZ}), single- and multi-answer multiple choice (\task{SC-MC},
\task{MC-MC}), term matching (\task{MATCH}), causal explanation
(\task{WHY}), definitions (\task{DEF}), multi-hop questions (\task{MHOP})
and methodological critiques (\task{ERR}); Table~\ref{tab:inventory}
(Appendix~\ref{app:details}) gives the inventory. We focus on the three
reasoning targets, \task{WHY}, \task{MHOP} and \task{ERR}, which lie outside
an untuned model's reach (Section~\ref{sec:intro}; Figure~\ref{fig:absolute},
Appendix~\ref{app:cells}).

\paragraph{Probe design.}
Four balanced designs each withhold a disjoint pair of targets and vary the
remaining six sources over mixture sizes $d=1,\dots,6$
(Appendix~\ref{app:details}). Across the four designs, each condition trains
$94$ distinct mixtures ($282$ over the three conditions). Pooling across
designs, $53$ of these mixtures withhold a given reasoning target, and among
them all seven other tasks appear as sources, the candidates when we later
select a mixture for that target. Training on all seven
in equal parts is the \emph{uniform mixture}, the choice of the count
heuristic and the default this paper measures against.
Each run holds exactly $6{,}000$ examples divided equally (up to rounding) among its
sources, uses one split manifest per corpus, and excludes the target task
from training. The passages behind different task types overlap only partially
(Section~\ref{sec:limitations}).

\paragraph{Training and scoring.}
All runs train one epoch of response-masked supervised fine-tuning with
rank-$16$ rank-stabilized LoRA (rsLoRA; \citealp{kalajdzievski2023rslora})
under a shared optimizer and decoding configuration
(Appendix~\ref{app:details}); only the constant-learning-rate runs of
Appendix~\ref{app:order} change the schedule. Closed tasks score by exact or
option-set match. Open tasks (\task{DEF}, \task{WHY}, \task{MHOP}, \task{ERR})
use an embedding rule. We split gold and generated answers into sentences, let
$s$ be the cosine similarity of \texttt{all-mpnet-base-v2} sentence
embeddings \citep{reimers2019sbert}, and count an item as correct when
\begin{equation}
  s_{\mathrm{full}}>\tau \quad\text{\textbf{or}}\quad s_{\mathrm{avg}}>\tau,
  \qquad \tau=0.85 .
  \label{eq:metric}
\end{equation}
Here $s_{\mathrm{full}}$ is the similarity of the whole answers, and
$s_{\mathrm{avg}}=\tfrac12\bigl[\tfrac1m\sum_i\max_j s(p_i,g_j)+\tfrac1n\sum_j\max_i s(p_i,g_j)\bigr]$
averages each sentence's best match on the other side, over the $m$ predicted
sentences $p_i$ and the $n$ gold sentences $g_j$.

\paragraph{Estimating the transfer map.}
Let $V$ be the set of eight task types. Write the accuracy of a withheld target
$T$ after training on a source set $M \subseteq V \setminus \{T\}$ as a set function
$y_T : 2^{V \setminus \{T\}} \to \mathbb{R}$. Any regression model of $y_T$
keeps some orders of interaction among sources and drops the rest. We fit the
additive (degree-one) case
\begin{equation}
  y_T(M) \;=\; \mu_T \;+\; \sum_{S \in V \setminus \{T\}} \beta_{S \to T}\, z_S
  \;+\; \varepsilon, \qquad z_S = \mathbb{1}[S \in M],
  \label{eq:additive}
\end{equation}
once per cell, weighting each distinct mixture once (retrainings are not
double counted). We report every coefficient as a \emph{centered substitution
effect}: within a target's fit we subtract the mean source coefficient, the
normalization we adopt because the design identifies only contrasts between
sources (Section~\ref{sec:sizes}). An entry then reads as the
change in the target's accuracy when that source replaces an average source,
and only after centering do two targets' columns mean the same thing. The
\emph{transfer map} holds these entries for every source--target pair.
A source \emph{helps} a target when its entry is positive and \emph{hurts} it
when the entry is negative.

\paragraph{From map to mixture.}
One rule turns a cell's fit into a mixture: keep every source whose uncentered
coefficient $\beta_{S \to T}$ is positive and drop the rest, with a predicted
effect over the uniform mixture equal to
minus the sum of the dropped coefficients. The rule uses uncentered coefficients
because each is the fit's predicted change from adding that source to a
mixture at a fixed budget. The
rule keeps three to seven of the seven candidate sources. The recorded
predictions of Section~\ref{sec:predict} apply it to fits over every run trained
before the test runs, retrainings included (Appendix~\ref{app:cells}).

\paragraph{Prediction baselines.}
Two baselines test the predictions. The \emph{mixture-agnostic baseline} predicts
that no change of mixture moves the target, and the \emph{count baseline} predicts
from mixture size alone, with a linear or a quadratic fit to the probe.

\paragraph{Competing selectors.}
Five selectors compete with the map's mixture at the same budget and one fixed
training seed. The uniform mixture trains all seven sources, and the single best
source scores highest on the target when trained alone in the probe. A
TaskShop-style ranker \citep{kim2023taskweb} keeps the sources with the highest
such scores; a Skill-It-style ranker \citep{chen2023skillit} first weights each
score by how well the source learns its own task. 
A third ranker, embedding similarity (the similarity heuristic run as a selector), 
keeps the sources whose mean embedding of
sampled question--answer pairs, under the scoring rule's encoder, lies closest to
the target's. The three rankers keep as many sources as the map's mixture, and we
also run TaskShop and embedding similarity with three and five sources
(Appendix~\ref{app:selectors}).
\section{The Transfer Map}
\label{sec:map}

\begin{figure}[t]
  \centering
  \includegraphics[width=\textwidth]{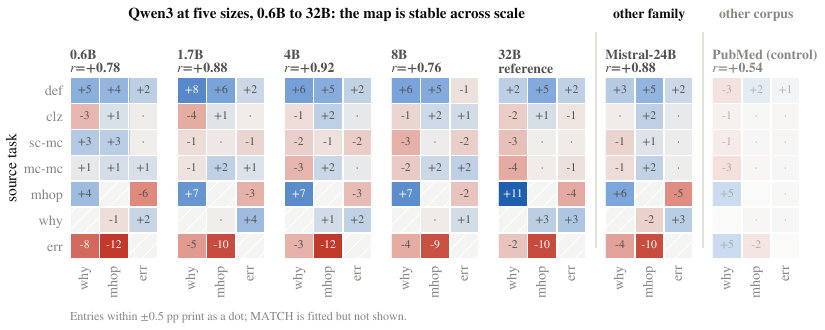}
  \caption{\textbf{The transfer map is signed, directed, and stable across
  model size.} Effect of each source (rows) on each withheld reasoning
  target (columns), in pp relative to an average source. Left
  five panels: Qwen3 at five sizes; $r$ is the mean per-target correlation with
  the $32$B map (Section~\ref{sec:scope}). Mistral-24B on the same corpus
  keeps the map; Qwen3-32B on PubMed (faded), with about a tenth of the
  evaluation items, agrees less, but \task{ERR} remains \task{MHOP}'s most
  negative source. For these two panels, $r$ pools the $18$ entries from sources other than
  \task{MATCH}, which scores near ceiling (Appendix~\ref{app:details}).}
  \label{fig:map}
\end{figure}

\subsection{Transfer is directed and signed}
\label{sec:asym}

The map is asymmetric as a whole (Table~\ref{tab:asym}). The off-diagonal map
correlates with its transpose at $r=0.304$, $0.290$ and $0.262$ in the three
conditions; a symmetric map would give $1$. Estimation
noise alone can lower $r$, so we test each of the $28$ task pairs directly. In
$14$, $13$ and $9$ of them, the bootstrap interval on the difference between the
two directions excludes zero. In $3$, $4$ and $0$ pairs, the two directions have
opposite signs and each direction's own interval excludes zero. The pair from the
introduction is one of them: on Qwen3-32B, \task{DEF}$\to$\task{MHOP} is
$+5.27\pp$ $[+3.55,+6.84]$ and \task{MHOP}$\to$\task{DEF} is $-1.24\pp$
$[-1.68,-0.75]$. Uncentered, the second coefficient is $-1.92\pp$
$[-2.38,-1.41]$: adding \task{MHOP} to a mixture at a fixed budget lowers
\task{DEF} accuracy in absolute terms. In all three
conditions, the most asymmetric pair is \task{MHOP} and \task{WHY}
(Figure~\ref{fig:overview}c; Appendix~\ref{app:asymmetry}).

The map also has large negative entries (Figure~\ref{fig:map}). On Qwen3-32B and Mistral-24B the most
negative entry is \task{ERR}$\to$\task{MHOP}, at $-9.72$ and $-9.55\pp$; we call a
target's most negative source its \emph{interferer}. These entries do not track
supervised token volume (Appendix~\ref{app:details}).

\subsection{Source identity, not mixture size, predicts accuracy}
\label{sec:sizes}

Table~\ref{tab:cv} (Appendix~\ref{app:identification}) compares how well source
identity and mixture size predict a target's accuracy on unseen mixtures
(Figure~\ref{fig:overview}b).
Source identity explains $0.768$, $0.746$ and $0.557$ of the within-target
variance in the three conditions. A linear term in mixture size explains
$0.114$, $0.003$ and $0.153$, and a quadratic term and an indicator for each
size both do worse. Adding all $21$ pairwise interactions between sources raises
the first two to $0.840$ and $0.849$ but lowers the third to $0.522$. Sources also
interact, but one entry per source already carries most of the predictable
variance.

The count heuristic rests on evidence that accuracy rises with the number of
tasks. Lemma~\ref{lem:alias} bounds what such evidence can reveal whenever each source
may have its own effect, as in the experiments behind it, and rests on one fact:
a mixture's size counts the sources it contains, $d=|M|=\sum_S z_S$. Add a size term $\gamma d$ to
Equation~\ref{eq:additive}:

\begin{lemma}[Size alias]
\label{lem:alias}
\textbf{(a)} For any $c$, replacing every $\beta_S$ by $\beta_S+c$ and $\gamma$
by $\gamma-c$ leaves every fitted value of $y_T$ unchanged, because
$\sum_S z_S=d$. The data can therefore identify only the contrasts
$\beta_S-\beta_{S'}$ and the sums $\beta_S+\gamma$, never $\gamma$ or any
$\beta_S$ alone. The same holds at every interaction order: adding $\gamma_k$
to every order-$k$ interaction coefficient absorbs a size term
$\gamma_k\binom{d}{k}$, because $\binom{d}{k}$ is the sum of
$z_{S_1}\cdots z_{S_k}$ over all sets of $k$ sources.

\textbf{(b)} Regressing $y_T$ on $d$ alone by ordinary least squares (OLS)
gives the slope
$\sum_S w_S\hat\beta_S$, where $\hat\beta_S$ are the OLS coefficients of
Equation~\ref{eq:additive} on the same runs and the weights
$w_S=\mathrm{Cov}(z_S,d)/\mathrm{Var}(d)$ sum to one. When every source appears
equally often at every size, the weights are equal and the slope is the mean
$\hat\beta_S$.
\end{lemma}
\noindent Appendix~\ref{app:identification} proves both parts.
By (a), ``more sources help'' and ``the average source helps'' are the same
claim, $\bar\beta+\gamma>0$ for the mean $\bar\beta$ of the $\beta_S$. By (b), a
positive size slope shows only that the average task helps and cannot justify
adding every task.

Accuracy on the withheld target rises with mixture size for the three reasoning
targets, except on Mistral/\task{MHOP} where it is flat, and falls for \task{DEF}
in both neuroscience conditions (Table~\ref{tab:dose}, Figure~\ref{fig:dose}).
No single size bonus exists. Even where the average source helps, some sources have
$\hat\beta_S<0$, and Section~\ref{sec:predict} shows that dropping them at a
fixed budget raises the target.
\section{Predicting Unseen Mixtures}
\label{sec:predict}

We test whether the map predicts the effect of moving from the uniform
mixture, which no probe run trains, to the mixture that the rule of Section~\ref{sec:setup} selects. We compare the two
at the same $6{,}000$-example budget, paired by data-order seed.

We recorded each predicted effect before training the $54$ test runs: both
mixtures at three data-order seeds in each of nine cells (the three reasoning
targets in each condition). In three cells the
probe already contains the selected mixture; refitting the map without those
runs moves their predictions by at most $0.22\pp$. One cell, PubMed/\task{ERR},
is vacuous: the rule keeps all seven sources, so it predicts no change. That
leaves $n=8$ cells (Table~\ref{tab:cells}, Appendix~\ref{app:cells}).

\paragraph{The recorded predictions outperform both baselines.}
We lead with mean absolute error (MAE), because at $n=8$ two large-effect cells
dominate Pearson $r$. The recorded predictions have an MAE of $2.31\pp$, with
$r=0.865$ and a slope of realized on predicted effect, fitted with no
intercept, of $1.054$ ($1$ is ideal; Table~\ref{tab:calib},
Figure~\ref{fig:overview}d). The mixture-agnostic baseline, which predicts zero
effect, has an MAE of $6.20\pp$. The count baseline does worse still, at
$9.31\pp$ (linear fit) and $9.20\pp$ (quadratic fit). Every selected mixture
has fewer sources than the uniform mixture, and accuracy on these targets
rises with mixture size on average (Section~\ref{sec:sizes}), so the linear
fit predicts a loss in every cell and the quadratic fit in seven of the eight.
Even the best constant prediction, chosen after seeing the outcomes, has an MAE
of $4.80\pp$ (Table~\ref{tab:calib}). The map thus tracks how the effect varies
across cells, beyond its typical size. With each target's mean removed from the
predicted and realized effects, its MAE is $0.79\pp$ against $2.26\pp$ for a
per-target constant (Appendix~\ref{app:cells}).

\paragraph{Most of the gain comes from dropping one task type.}
Six of the eight realized effects are gains, of up to $13.76\pp$ over three
seeds. On critiques, the uniform mixture gains only about $1\pp$ over the untuned
model in both neuroscience conditions, and the selected mixture doubles its
accuracy (Figure~\ref{fig:absolute}, Appendix~\ref{app:cells}). In each non-vacuous \task{MHOP} and
\task{ERR} cell, the target's interferer carries $75$ to $100\%$ of the predicted
gain.

\paragraph{The misses.}
In two cells the map predicted a small gain and the realized effect was
slightly negative: Qwen/\task{WHY} ($+2.64$ predicted, $-0.90$ realized) and
PubMed/\task{WHY} ($+1.18$, $-0.95$). Neither loss has a 95\% confidence
interval that excludes zero, whereas all six gains do (Table~\ref{tab:cells}).
These two cells therefore neither confirm nor contradict the predicted sign.
\section{The Map Against Other Selectors}
\label{sec:buys}

Table~\ref{tab:selectors} compares the mixture the map selects against the five
selectors of Section~\ref{sec:setup} at one training seed, each at its own mixture
size $d$, the count of sources it keeps (the map's $d$ for the three rankers).
The similarity selector loses $7.47$ to
$9.68\pp$ on three of the four \task{MHOP} and \task{ERR} cells of Qwen3-32B and
Mistral-24B and $3.11\pp$ on the fourth (Figure~\ref{fig:selectors}a,
Appendix~\ref{app:selectors}). Like the map, our TaskShop-style ranker ranks sources by their measured
effect on the target. It stays within $2.73\pp$ of the map on the six neuroscience cells, so the
similarity selector loses because it does not measure each source's effect. The Skill-It-style ranker edges the map on three
cells, by at most $1.45\pp$, and trails it on the other three, by $2.43$ to $7.91\pp$ (its Qwen/\task{MHOP} run tripped the degeneration check). Each comparison uses
one run per mixture, so the TaskShop margins and Skill-It's edges are not resolved.

\begin{table}[t]
\centering
\caption{Map-selected mixture minus competitor on the withheld target, at
each competitor's mixture size $d$ and at one training seed; positive
favours the map. ``Interferer'' is the target's most negative measured
source. $\dagger$ marks a run whose generations tripped a degeneration check
(Appendix~\ref{app:details}). $\ddagger$ marks the vacuous cell: every
uncentered coefficient there is positive, so the rule keeps all seven sources
and we print no interferer; its two runs train the same mixture at the same
seed, so their $+1.69\pp$ gap is one replicate difference. Full sweep in
Appendix~\ref{app:selectors}.}
\label{tab:selectors}
\footnotesize
\begin{tabular}{llrrrrrr}
\toprule
Condition & Target & Interferer & Uniform & Single-best & Skill-It & TaskShop & Embed-sim \\
\midrule
Qwen3-32B   & \task{WHY}  & \task{MC-MC} & $+1.21$  & $-2.37$          & $-1.45$          & $+0.00$ & $-0.53$ \\
Qwen3-32B   & \task{MHOP} & \task{ERR}   & $+13.19$ & $-1.41^{\dagger}$& $+2.43^{\dagger}$& $+2.73$ & $+3.11$ \\
Qwen3-32B   & \task{ERR}  & \task{MHOP}  & $+8.64$  & $+2.11^{\dagger}$& $-0.20$          & $+0.64$ & $+8.10$ \\
Mistral-24B & \task{WHY}  & \task{ERR}   & $+2.18$  & $-2.13$          & $-1.31$          & $-0.34$ & $+1.36$ \\
Mistral-24B & \task{MHOP} & \task{ERR}   & $+17.23$ & $+0.44^{\dagger}$& $+7.59$          & $-0.15$ & $+9.68$ \\
Mistral-24B & \task{ERR}  & \task{MHOP}  & $+7.76$  & $+1.42$          & $+7.91$          & $+0.10$ & $+7.47$ \\
PubMed      & \task{WHY}  & \task{MC-MC} & $+0.00$  & $+4.27$          & $+17.54$         & $-1.42$ & $+1.90$ \\
PubMed      & \task{MHOP} & \task{ERR}   & $+2.45$  & $+4.90$          & $+1.47$          & $+0.98$ & $+1.96$ \\
PubMed      & \task{ERR}  & ---          & $+1.69^{\ddagger}$ & $+0.00$  & ---              & --- & --- \\
\bottomrule
\end{tabular}
\end{table}

\paragraph{Similarity cannot separate a helpful source from a harmful one.}
A single matched $d$ leaves open whether that $d$ favoured the map, so
we re-run the similarity selector at $d=3$, at its matched $d$ and at
$d=5$. The map stays ahead at every $d$ on all four
\task{MHOP} and \task{ERR} cells of Qwen3-32B and Mistral-24B, by $3.11$ to
$15.91\pp$, and within $1.45\pp$ of zero on their \task{WHY} cells
(Appendix~\ref{app:selectors}). Both large jumps in that
sweep occur on the \task{MHOP} cells, when widening the mixture to $d=5$
admits the target's interferer, \task{ERR}.
A similarity score does not encode the sign of transfer, so a source that helps a target and a
source that hurts it enter such a selector alike: at its matched
$d$ the similarity mixture holds the target's interferer in four of
the eight non-vacuous cells. Nor is the gap a matter of examples per source,
which rise as a mixture narrows: at $d=3$ the similarity selector trains at
least as many examples per source as the map on the same four \task{MHOP} and \task{ERR} cells, and
still trails it by $3.41$ to $9.68\pp$ (Appendix~\ref{app:selectors}).
\section{Where the Map Transfers: Target, Corpus, Family, and Scale}
\label{sec:scope}

\paragraph{A selection is specific to its target.}
The selection rule (Section~\ref{sec:setup}) keeps a different set of sources for
each target in the same condition (Table~\ref{tab:cells}), and raising accuracy on
the target costs accuracy on the other tasks. Across the eight non-vacuous cells, the
map's mixture scores on average $2.40\pp$ below the uniform mixture on the non-target
macro (the mean accuracy over the six scored tasks other than the target and
\task{MATCH}) and $2.70\pp$ above the single best source
(Figure~\ref{fig:selectors}b, Appendix~\ref{app:selectors}).

\paragraph{The map agrees across model families but did not carry over to our second corpus.}
Over the $18$ map entries from a non-\task{MATCH} source to a reasoning target, the
Qwen3-32B and Mistral-24B maps agree at $r=0.880$, while the PubMed map agrees with
them at only $r=0.279$ to $0.545$. The PubMed condition also changes document
granularity, pool size and evaluation size (Section~\ref{sec:setup}), so this
comparison does not isolate the corpus. When we refit the map without one
design's mixtures and let it pick among that design's trained mixtures
(out-of-fold), its pick beats the design's widest mixture by $+5.03\pp$ on Qwen3-32B
and $+5.82\pp$ on Mistral-24B but trails it by $0.51\pp$ on PubMed. These are
retrospective gains on folds that share no mixture.

\paragraph{The map is stable across model scale.}
We fit the map independently at four other Qwen3 sizes, $0.6$B to $8$B, each from
the same fixed $24$-run probe, a reduced design with mixture sizes $1$ to $4$
(Appendix~\ref{app:details}). Every map agrees with the $32$B map, at $r=0.757$ to $0.918$
(Figure~\ref{fig:map}). Of the $21$ entries from a source to a reasoning target,
$11$ keep their sign at all five sizes and on Mistral-24B. \task{DEF} helps
\task{MHOP} by $4.4$ to $5.8\pp$, and \task{ERR} hurts it by $9.3$ to $11.6\pp$. Agreement does not track size: the $8$B map agrees least. The
degeneration check flags $18$ of the $24$ probe runs at $0.6$B, mostly for cloze
answers that hit the token cap. Across the probes at the four smaller sizes, only one flag, on
\task{WHY} at $1.7$B, falls on a column that a reasoning-target fit reads
(Appendix~\ref{app:details}).

\paragraph{A selection made at one size beats the uniform mixture at every other size.}
Before training at any other size, we recorded the mixture that the $32$B map
selected for each target, then trained it against the uniform mixture at each of
the four other sizes over three seeds (Table~\ref{tab:xsize}, Appendix~\ref{app:cells}). These $12$
size--target cells form a second test of recorded predictions, disjoint from the eight cells of
Section~\ref{sec:predict}. The selected mixture wins in $11$ of the $12$ cells
(seven intervals exclude zero, none below it), by a mean of $+4.44\pp$, and the
$32$B map predicts these effects with an MAE of $2.28\pp$, against $4.64\pp$ for the
mixture-agnostic baseline. The mean gain is positive at every size, from $+3.21$ to
$+5.26\pp$; \task{WHY} is the weakest target at every size and accounts for the one
loss. A map refit at each size from its $24$-run probe gains $+6.53\pp$ over the
same $12$ cells, and every interval excludes zero; the transferred selection keeps
about two thirds of that gain.

\paragraph{A $24$-run probe selects nearly as well as the full probe.}
When each out-of-fold refit sees only $24$ randomly drawn probe runs, the gain on
Qwen3-32B is $+4.83\pp$, or $96\%$ of the full probe's $+5.03\pp$; on
Mistral-24B, $24$ runs recover the full gain ($+5.88$ against $+5.82\pp$,
Table~\ref{tab:scope}). The gain varies little across subsets (standard deviation $0.33\pp$ on Qwen3-32B,
Appendix~\ref{app:details}). A $24$-run probe costs about $22$ GPU-hours at $32$B and
$11$ to $16$ at the other sizes, a one-time cost per condition.

\section{Limitations}
\label{sec:limitations}
 
\paragraph{Every score comes from an automated rule.}
No human judges answer quality, so every effect here is an effect on
our automated scoring rules. The gold answers are themselves generated by
GPT-5-mini, so accuracy measures agreement with generated references under the
rule. Without the rule's threshold, every realized effect of
Section~\ref{sec:predict} is positive (Appendix~\ref{app:metric}). Restricted to single-sentence answers,
which holds the answer's format fixed, the two largest effects survive and two
others vanish or reverse; in those two cells one run rarely answers in one sentence. A
human evaluation of the predicted cells would settle whether the effects hold beyond the rule.
 
\paragraph{The design isolates neither the passages nor the dose of a source.}
The passages behind different task types overlap only in part, so this design cannot isolate
supervision format. At a fixed budget, a narrower mixture
also trains more examples per source, so the identity of a source and its dose move
together. Supervised token volume does not rank the effects
(Appendix~\ref{app:details}), but no run holds passages or examples per source
fixed; passage-matched and dose-matched controls would settle both.
 
\paragraph{The prediction tests cover few cells.}
Each test, of eight and of twelve cells, reuses targets and conditions
across its cells, so its intervals are descriptive, and we transferred selections
across scale within one model family. The PubMed condition has about a tenth of the
evaluation items, so we report its cells as unresolved and not as null.
 
\paragraph{The mixture is a set, but training it is a sequence.}
Equation~\ref{eq:additive} treats a mixture as a set of sources. In a separate
experiment, the order of an identical set of examples alone moves the target's
accuracy where an interferer is among the sources (Appendix~\ref{app:order}, which
also reports a recorded prediction that failed). We read this as the boundary of
the set abstraction.

\section{Conclusion}
\label{sec:conclusion}

In this paper we asked which task types to train on when adapting a language
model to a specialized corpus under a fixed budget. The count heuristic and the
similarity heuristic rest on assumptions that fail. Transfer can be negative: in both
neuroscience conditions, one interfering task type cancels most of what the
other types add on two of the three reasoning targets. It is not symmetric: on
both model families, some task pairs help in one direction and hurt in the
other. Which sources one adds predicts accuracy; how many does not, because no
fit can separate mixture size from the average source effect.

The transfer map replaces both assumptions with a measurement that a
practitioner makes once per corpus. It predicts a target's accuracy on unseen
mixtures before training, with less than half the error of a mixture-agnostic
baseline in both prediction tests. The choice matters: on Qwen3-32B, the map's
mixture raises multi-hop accuracy from $22\%$ to $34\%$ and doubles critique
accuracy, where the uniform mixture barely beats the untuned model. As a
selector, it beats the similarity selector at every mixture size we tried on the
neuroscience multi-hop and critique targets, because a similarity score carries
no sign. At each smaller model size, a map fitted from only $24$ runs
selects a mixture that beats training on all source tasks, and so does the
mixture the $32$B map selected.%
\label{endofbody:check}

Transfer is thus a property of the data: the map held when we changed the
model, across two families and a $53$-fold range of size, and moved when we
changed the corpus. A selection serves one target at some cost to the other
tasks, and a new corpus needs its own probe. Our second corpus also had much
smaller training pools and held-out split, and we could not resolve the probe's
gain there. Which corpus factors move the map is therefore open, as is what
makes a task type an interferer. The same choice arises in multi-domain reinforcement
learning, where training on one domain often degrades others
\citep{yang2026perturbation,mcpo2026,onestep2026}. Wherever task types share a
budget, a practitioner has a third option beyond adding every task or picking
the similar ones: measuring which task types help a target and which one
interferes, which here took two dozen fine-tuning runs.%
\label{endofbody}

\section*{AI Use Statement}
GPT-5-mini generated the instruction--response pairs of all eight task types, for
training and evaluation, from the corpus passages (Section~\ref{sec:setup}). We also
used AI assistants, including Claude, to help write code and to draft and edit the
text. We take responsibility for the final content of this work, including text,
claims and artifacts produced with the aid of generative AI.

\section*{Reproducibility Statement}
\ifarxiv
The repository \url{https://github.com/Vahidrostami/directed-transfer} contains the
analysis code and the frozen evidence bundle behind every table and figure.
\else
The anonymized supplementary material contains the analysis code and the frozen
evidence bundle behind every table and figure.
\fi
Section~\ref{sec:setup} and
Appendix~\ref{app:details} give the corpus split, task inventory, probe designs,
training configuration, decoding and scoring rule, and the commands that recompute
and verify every table and figure.

\ifarxiv
\section*{Acknowledgments}
This work was funded by the Deutsche Forschungsgemeinschaft (DFG, German Research
Foundation) through the Collaborative Research Center ``Motor Control in Health and
Disease'' (DFG-SFB~1451, Project INF, ID 431549029). We also thank the IT
Center University of Cologne (ITCC) for providing compute resources and support on
the DFG-funded HPC system RAMSES (Research Accelerator for Modeling and Simulation
with Enhanced Security) (DFG funding number: INST 216/512-1 FUGG).
\fi

\bibliography{references}
\bibliographystyle{iclr2027_conference}

\clearpage
\appendix
\section{Recorded Predictions and Where the Map Holds}
\label{app:cells}

\begin{table}[h]
\centering
\caption{Calibration on the same eight cells. The constant oracle is fitted on
the values it is scored against, so it lower-bounds any constant predictor.}
\label{tab:calib}
\footnotesize
\begin{tabular}{lrrr}
\toprule
Predictor & MAE (pp) & Pearson $r$ & Origin slope \\
\midrule
Recorded map prediction          & $\mathbf{2.31}$ & $0.865$ & $1.054$ \\
Zero effect                      & $6.20$ & --- & --- \\
Best constant, fitted post hoc   & $4.80$ & --- & $1.037$ \\
Count, linear                    & $9.31$ & $0.731$ & $-0.371$ \\
Count, quadratic                 & $9.20$ & $0.362$ & $-0.730$ \\
\bottomrule
\end{tabular}
\end{table}

\begin{table}[h]
\centering
\caption{Where the map holds. Scale agreement is the mean per-target centered
correlation with the Qwen3-32B map; transport agreement is the centered
correlation between conditions over the $18$ non-\task{MATCH}
source-to-reasoning entries. An out-of-fold gain refits the map without one
design's mixtures, lets it pick among that design's trained mixtures, and
compares the pick with that design's widest mixture; the $24$-run rows refit
on $24$ runs drawn at random from the same fitting runs, averaged over $60$
draws. The last block trains the recorded
$32$B selection against the uniform mixture at each size, averaged over the
three targets (Table~\ref{tab:xsize} gives every cell).}
\label{tab:scope}
\footnotesize
\begin{tabular}{lr@{\qquad}lr}
\toprule
\multicolumn{2}{l}{\emph{Scale: agreement with 32B map}} & \multicolumn{2}{l}{\emph{Transport: between conditions}} \\
\midrule
Qwen3-0.6B & $0.780$ & Qwen $\leftrightarrow$ Mistral & $0.880$ \\
Qwen3-1.7B & $0.876$ & Qwen $\leftrightarrow$ PubMed  & $0.545$ \\
Qwen3-4B   & $0.918$ & Mistral $\leftrightarrow$ PubMed & $0.279$ \\
Qwen3-8B   & $0.757$ & & \\
\midrule
\multicolumn{4}{l}{\emph{Out-of-fold gain over the widest mixture of the left-out design (pp)}} \\
\midrule
Qwen3-32B, full probe & $+5.03$ & Qwen3-32B, $24$ runs & $+4.83$ \\
Mistral-24B, full probe & $+5.82$ & Mistral-24B, $24$ runs & $+5.88$ \\
PubMed, full probe & $-0.51$ & PubMed, $24$ runs & $-0.27$ \\
\midrule
\multicolumn{4}{l}{\emph{Recorded $32$B selection, trained at each size (pp over uniform)}} \\
\midrule
Qwen3-0.6B & $+4.34$ & Qwen3-4B & $+5.26$ \\
Qwen3-1.7B & $+4.96$ & Qwen3-8B & $+3.21$ \\
\bottomrule
\end{tabular}
\end{table}

\begin{table}[h]
\centering
\caption{The second test of recorded predictions, cell by cell. \emph{Transported}: the
mixture the Qwen3-32B map selected for each target, recorded before any
other size ran, minus the uniform mixture; it keeps the Qwen3-32B sources of
Table~\ref{tab:cells} (\task{WHY}: \task{CLZ}, \task{DEF}, \task{ERR},
\task{MHOP}; \task{MHOP}: \task{CLZ}, \task{DEF}, \task{MATCH}, \task{WHY};
\task{ERR}: \task{DEF}, \task{MATCH}, \task{SC-MC}, \task{WHY}), and
\emph{Recorded} is the $32$B map's prediction. \emph{Refit}: the mixture
that the map refitted at that size selects, with $d$ sources, minus the
uniform mixture. Effects are paired over three data-order seeds, with $t$
intervals over the seed-paired differences. No run in these cells trips the
degeneration check on the withheld target except in the refit
1.7B/\task{MHOP} cell.}
\label{tab:xsize}
\footnotesize\setlength{\tabcolsep}{4pt}
\begin{tabular}{llrrlrrl}
\toprule
& & \multicolumn{3}{c}{Transported $32$B selection} & \multicolumn{3}{c}{Refit at this size} \\
\cmidrule(lr){3-5}\cmidrule(lr){6-8}
Size & Target & Recorded & Realized & 95\% CI & $d$ & Realized & 95\% CI \\
\midrule
0.6B & \task{WHY} & $+2.64$ & $+1.39$ & $[-2.41,+5.19]$ & $4$ & $+7.15$ & $[+5.20,+9.11]$ \\
0.6B & \task{MHOP} & $+9.92$ & $+5.04$ & $[+4.38,+5.71]$ & $3$ & $+6.60$ & $[+5.41,+7.79]$ \\
0.6B & \task{ERR} & $+3.53$ & $+6.60$ & $[+4.77,+8.43]$ & $3$ & $+6.88$ & $[+4.97,+8.78]$ \\
1.7B & \task{WHY} & $+2.64$ & $+1.49$ & $[-0.07,+3.04]$ & $4$ & $+6.10$ & $[+5.50,+6.71]$ \\
1.7B & \task{MHOP} & $+9.92$ & $+7.69$ & $[+5.02,+10.36]$ & $4$ & $+8.42$ & $[+7.19,+9.65]$ \\
1.7B & \task{ERR} & $+3.53$ & $+5.71$ & $[+5.25,+6.18]$ & $3$ & $+5.14$ & $[+3.89,+6.39]$ \\
4B & \task{WHY} & $+2.64$ & $+0.34$ & $[-1.84,+2.52]$ & $3$ & $+3.92$ & $[+1.95,+5.89]$ \\
4B & \task{MHOP} & $+9.92$ & $+10.61$ & $[+7.38,+13.83]$ & $5$ & $+10.14$ & $[+8.01,+12.27]$ \\
4B & \task{ERR} & $+3.53$ & $+4.83$ & $[+2.92,+6.74]$ & $4$ & $+5.04$ & $[+3.89,+6.20]$ \\
8B & \task{WHY} & $+2.64$ & $-1.16$ & $[-2.43,+0.11]$ & $3$ & $+3.15$ & $[+1.30,+5.00]$ \\
8B & \task{MHOP} & $+9.92$ & $+6.34$ & $[+2.16,+10.52]$ & $4$ & $+11.37$ & $[+7.32,+15.42]$ \\
8B & \task{ERR} & $+3.53$ & $+4.45$ & $[-0.28,+9.19]$ & $3$ & $+4.39$ & $[+3.03,+5.75]$ \\
\bottomrule
\end{tabular}
\end{table}

\begin{table}[h]
\centering
\caption{Every predicted cell, recorded prediction against realized effect: selected mixture
minus uniform default, paired over three data-order seeds, with $t$ intervals
over the seed-paired differences. $d$ is the number of sources the rule keeps,
against seven in the uniform default, so a smaller $d$ also means more examples
per source. The PubMed \task{ERR} cell is excluded from the primary $n=8$:
every uncentered coefficient there is positive, the rule drops nothing, and the predicted
effect is exactly zero. The probe already trains the selected mixture of
Mistral/\task{MHOP}, Mistral/\task{ERR} and PubMed/\task{MHOP}; refitting
without those runs moves their predictions by at most $0.22\pp$.}
\label{tab:cells}
\footnotesize\setlength{\tabcolsep}{4pt}
\begin{tabular}{ll>{\raggedright\arraybackslash}p{3.3cm}rrrr}
\toprule
Condition & Target & Sources kept & $d$ & Recorded & Realized & 95\% CI \\
\midrule
Qwen3-32B   & \task{WHY}  & \task{CLZ}, \task{DEF}, \task{ERR}, \task{MHOP} & $4$ & $+2.64$  & $-0.90$  & $[-4.73,+2.92]$ \\
Qwen3-32B   & \task{MHOP} & \task{CLZ}, \task{DEF}, \task{MATCH}, \task{WHY} & $4$ & $+9.92$  & $+11.89$ & $[+6.99,+16.79]$ \\
Qwen3-32B   & \task{ERR}  & \task{DEF}, \task{MATCH}, \task{SC-MC}, \task{WHY} & $4$ & $+3.53$  & $+7.96$  & $[+6.76,+9.16]$ \\
Mistral-24B & \task{WHY}  & \task{CLZ}, \task{DEF}, \task{MC-MC}, \task{MHOP}, \task{SC-MC} & $5$ & $+3.79$  & $+2.49$  & $[+1.16,+3.81]$ \\
Mistral-24B & \task{MHOP} & \task{CLZ}, \task{DEF}, \task{MC-MC} & $3$ & $+14.86$ & $+13.76$ & $[+3.78,+23.73]$ \\
Mistral-24B & \task{ERR}  & all but \task{MHOP} & $6$ & $+4.95$  & $+8.56$  & $[+6.65,+10.48]$ \\
PubMed      & \task{WHY}  & \task{CLZ}, \task{ERR}, \task{MHOP} & $3$ & $+1.18$  & $-0.95$  & $[-9.84,+7.94]$ \\
PubMed      & \task{MHOP} & \task{CLZ}, \task{DEF}, \task{MATCH}, \task{MC-MC}, \task{SC-MC} & $5$ & $+2.68$  & $+3.10$  & $[+0.29,+5.92]$ \\
\midrule
PubMed      & \task{ERR}  & all seven & $7$ & $+0.00$  & \multicolumn{2}{l}{excluded: vacuous} \\
\bottomrule
\end{tabular}
\end{table}

A predictor that knows only which targets respond strongly is a per-target
constant, and removing each target's mean from both columns sends it to zero.
What remains is the variation across conditions within a target, and there the
recorded predictions still track the realized effects at $r=0.951$ with MAE
$0.79\pp$, against $2.26\pp$ for a per-target constant. Three targets over eight cells leave five effective
degrees of freedom, and centering makes the two-cell \task{ERR} group
anticorrelate by construction. The check therefore rules target identity out
as the carrier without being an independent test.
Table~\ref{tab:cv} is the within-target test, over $53$ mixtures per cell.

The recorded selections came from fits with one row per run, retrainings
included. Refitting each cell with one row per distinct mixture keeps the same
sources in six of the nine cells. Each of the three differences turns on
coefficients within $0.5\pp$ of zero: Mistral/\task{WHY} would drop
\task{SC-MC}, PubMed/\task{WHY} would add \task{MATCH} and \task{SC-MC}, and
PubMed/\task{MHOP} would swap \task{MC-MC} for \task{WHY}.

Figure~\ref{fig:absolute} puts the same cells on an absolute scale.

\begin{figure}[h]
  \centering
  \includegraphics[width=\textwidth]{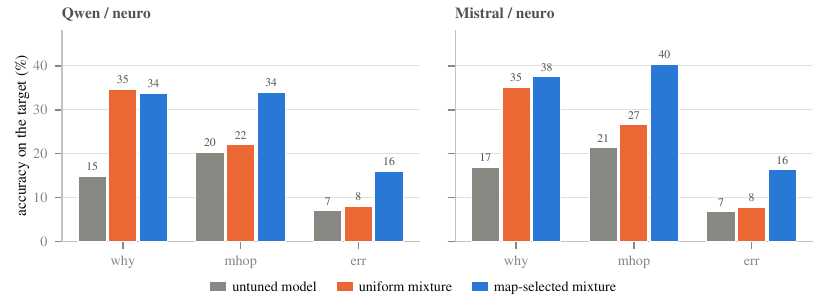}
  \caption{\textbf{The effects on an absolute scale.} Untuned base, the uniform
  mixture and the map-selected mixture on each withheld target, for Qwen3-32B
  and Mistral-24B, without the PubMed control. Where
  uniform already captures most of the available gain (\task{WHY}) the map adds
  little; where uniform barely improves on the untuned model (\task{MHOP},
  \task{ERR}) it adds the most.}
  \label{fig:absolute}
\end{figure}

\clearpage
\section{Identification}
\label{app:identification}

\begin{proof}[Proof of Lemma~\ref{lem:alias}]
For (a), for every source set $M$ with $d=|M|$,
\begin{align*}
\sum_S(\beta_S+c)z_S+(\gamma-c)d
&=\sum_S\beta_Sz_S+c\sum_Sz_S+\gamma d-cd\\
&=\sum_S\beta_Sz_S+\gamma d ,
\end{align*}
because $\sum_S z_S = d$. Fitted values and likelihood are therefore
identical for every $c$, so $\gamma$ is not separately estimable. Centering
subtracts the same $c$ from the coefficient mean, which is why the centered
substitution effects reported throughout are invariant to the choice.

For (b), the ordinary least squares (OLS) residuals of Equation~\ref{eq:additive} are orthogonal to
every $z_S$ and hence to $d$, so
$\mathrm{Cov}(y_T,d)=\sum_S\hat\beta_S\,\mathrm{Cov}(z_S,d)$. Dividing by
$\mathrm{Var}(d)=\sum_S\mathrm{Cov}(z_S,d)$ gives the slope and shows that the
weights sum to one.
\end{proof}

Figure~\ref{fig:dose} plots the measured slope against the mean source effect.
The alias recurs at every order. Since
$\sum_{S<S'}z_Sz_{S'}=d(d-1)/2$, a quadratic polynomial in $d$ absorbs a common
shift of every pairwise coefficient; more generally the order-$k$
elementary symmetric sum of binary presence indicators equals $\binom{d}{k}$, so
no fit separates a generic degree-$k$ size component from the mean order-$k$
interaction without an external restriction, and the data identify only
\emph{heterogeneous} nonadditivity.

This leaves Table~\ref{tab:cv} untouched, because it fits count-only and
presence-only representations separately and compares their leave-one-out predictions
without interpreting a count slope beside presence indicators. The data
identify source-specific pairwise interactions, and these improve leave-one-out
prediction on
both neuroscience conditions (Table~\ref{tab:cv}). Selecting from the pairwise
fit changes the out-of-fold gain of Section~\ref{sec:scope} from $+5.03$ to
$+5.37\pp$ on Qwen3-32B, from $+5.82$ to $+4.93\pp$ on Mistral-24B and from
$-0.51$ to $-0.28\pp$ on PubMed, so the interactions do not make a consistently
better selector.

\begin{table}[h]
\centering
\caption{Leave-one-mixture-out prediction of the withheld target,
between-target variance removed (Section~\ref{sec:sizes}). MAE is against a
per-target constant scoring $5.16/4.05/2.81\pp$. Mixture size gets every
advantage: linear, quadratic, and a free indicator per size. \emph{Share} adds
each source's share of the training examples, $z_S/d$. \emph{Pairwise} adds all
$21$ products $z_S z_{S'}$: $29$ parameters against $52$ training mixtures per
fold.}
\label{tab:cv}
\footnotesize
\begin{tabular}{lrrrrrr}
\toprule
& \multicolumn{2}{c}{Qwen3-32B} & \multicolumn{2}{c}{Mistral-24B} & \multicolumn{2}{c}{PubMed} \\
\cmidrule(lr){2-3}\cmidrule(lr){4-5}\cmidrule(lr){6-7}
Model of the mixture & $R^2$ & MAE & $R^2$ & MAE & $R^2$ & MAE \\
\midrule
Size, linear        & $0.114$ & $4.66$ & $0.003$  & $3.92$ & $0.153$ & $2.52$ \\
Size, quadratic     & $0.085$ & $4.75$ & $-0.027$ & $3.96$ & $0.121$ & $2.57$ \\
Size, categorical   & $-0.019$& $5.04$ & $-0.130$ & $4.19$ & $0.037$ & $2.67$ \\
Source presence     & $0.768$ & $2.44$ & $0.746$ & $1.82$ & $\mathbf{0.557}$ & $1.95$ \\
Presence $+$ share  & $0.744$ & $2.29$ & $0.702$  & $1.80$ & $0.523$ & $1.98$ \\
Presence $+$ pairwise & $\mathbf{0.840}$ & $1.73$ & $\mathbf{0.849}$ & $1.34$ & $0.522$ & $1.95$ \\
\bottomrule
\end{tabular}
\end{table}

\begin{table}[h]
\centering
\caption{The dose response is the mean source effect. The slope in $d$ is
OLS over the unique mixtures that withhold the target; $\overline{\beta}$
averages the presence-model coefficients from the same fit. Across all
$12$ cells they agree at $r=0.898$, mean absolute
difference $0.46\pp$, maximum $1.32\pp$.
\task{DEF} is included to show the slope is target-specific and can be
negative.}
\label{tab:dose}
\footnotesize
\begin{tabular}{llcrrr}
\toprule
Condition & Target & Accuracy at $d=1$\,/\,\dots\,/\,$6$ & Slope in $d$ & $\overline{\beta}$ & Diff. \\
\midrule
Qwen3-32B & \task{WHY} & $21.3$\,/\,$24.0$\,/\,$28.5$\,/\,$30.5$\,/\,$32.5$\,/\,$36.4$ & $+2.91$ & $+2.40$ & $+0.51$ \\
Qwen3-32B & \task{MHOP} & $24.7$\,/\,$24.9$\,/\,$26.2$\,/\,$28.2$\,/\,$28.2$\,/\,$32.1$ & $+1.15$ & $-0.17$ & $+1.32$ \\
Qwen3-32B & \task{ERR} & $8.8$\,/\,$9.9$\,/\,$10.2$\,/\,$11.9$\,/\,$12.4$\,/\,$14.6$ & $+0.96$ & $+0.40$ & $+0.56$ \\
Qwen3-32B & \task{DEF} & $32.0$\,/\,$31.4$\,/\,$30.6$\,/\,$29.5$\,/\,$28.8$\,/\,$29.0$ & $-0.80$ & $-0.68$ & $-0.13$ \\
\addlinespace
Mistral-24B & \task{WHY} & $28.2$\,/\,$30.1$\,/\,$33.1$\,/\,$33.3$\,/\,$34.4$\,/\,$35.1$ & $+1.50$ & $+1.29$ & $+0.21$ \\
Mistral-24B & \task{MHOP} & $31.2$\,/\,$30.9$\,/\,$31.1$\,/\,$31.1$\,/\,$30.9$\,/\,$33.7$ & $+0.06$ & $-1.26$ & $+1.32$ \\
Mistral-24B & \task{ERR} & $10.3$\,/\,$11.2$\,/\,$10.9$\,/\,$12.4$\,/\,$12.9$\,/\,$15.9$ & $+0.70$ & $-0.02$ & $+0.72$ \\
Mistral-24B & \task{DEF} & $29.5$\,/\,$29.8$\,/\,$28.7$\,/\,$27.8$\,/\,$28.1$\,/\,$28.9$ & $-0.47$ & $-0.56$ & $+0.09$ \\
\addlinespace
PubMed & \task{WHY} & $28.7$\,/\,$31.9$\,/\,$33.4$\,/\,$35.0$\,/\,$38.4$\,/\,$39.3$ & $+2.16$ & $+1.89$ & $+0.27$ \\
PubMed & \task{MHOP} & $33.0$\,/\,$33.1$\,/\,$33.8$\,/\,$34.8$\,/\,$33.9$\,/\,$38.7$ & $+0.49$ & $+0.17$ & $+0.31$ \\
PubMed & \task{ERR} & $8.1$\,/\,$8.4$\,/\,$9.2$\,/\,$9.8$\,/\,$10.3$\,/\,$9.3$ & $+0.53$ & $+0.51$ & $+0.02$ \\
PubMed & \task{DEF} & $26.3$\,/\,$26.9$\,/\,$26.9$\,/\,$27.2$\,/\,$26.2$\,/\,$26.6$ & $-0.02$ & $-0.00$ & $-0.02$ \\
\bottomrule
\end{tabular}
\end{table}

The $d=6$ column is the design's own six-source pool, which omits the target
and its design partner. For \task{MHOP} and \task{ERR} on
Qwen3-32B and Mistral-24B that partner is the target's interferer, so the $d=6$
accuracy sits above the seven-source uniform mixture: Qwen/\task{MHOP} reads
$32.1\%$ at $d=6$ against $22.0\%$ for the uniform run. The dose curve is
therefore not a ceiling for the uniform mixture, and the comparison that
matters for selection is the one of Section~\ref{sec:predict}, at matched
budget with the same seven candidates.

The correlation runs across cells spanning $-0.80$ (\task{DEF}) to $+2.91\pp$
(Qwen/\task{WHY}), and the two \task{MHOP} cells differ in sign while agreeing
in magnitude to within $1.32\pp$. The slope exceeds $\overline{\beta}$ in $10$ of
the $12$ cells. The dose curve is an unweighted marginal over mixtures that are
not balanced across $d$, so the two need not agree exactly.

\begin{figure}[h]
  \centering
  \includegraphics[width=\textwidth]{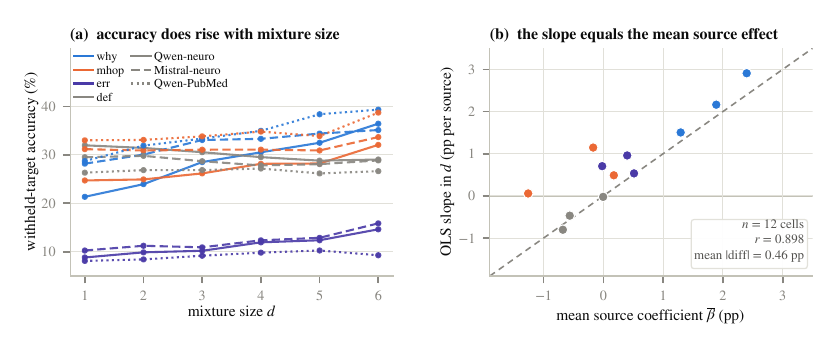}
  \caption{\textbf{Lemma~\ref{lem:alias} as a measurement.}
  \textbf{(a)} Withheld-target accuracy against mixture size. It rises for the
  reasoning targets except Mistral/\task{MHOP}, which is flat, and falls for
  \task{DEF} in both neuroscience conditions. \textbf{(b)} Each cell's OLS slope
  in $d$ against the mean source coefficient of the presence fit for the same
  cell. The points lie near the identity line: the quantity one would report as a
  breadth bonus is the average source effect.}
  \label{fig:dose}
\end{figure}

\clearpage
\section{Directedness of the Map}
\label{app:asymmetry}

Recomputed by \texttt{scripts/run\_asymmetry.py} from the same distinct-mixture
fits reported in Section~\ref{sec:map}, so the coefficients here and there are
one fit. The null tested is symmetry itself
($\beta_{S\to T}=\beta_{T\to S}$); intervals are the percentile bootstrap over
distinct mixtures with a fixed seed. A permutation null over off-diagonal
entries is the wrong reference for this question. Shuffling makes the two
directions independent ($r\approx0$), so it tests whether the map is more
symmetric than noise. Our question is whether the map departs from symmetry. The pairwise intervals carry the claim.

\begin{table}[h]
\centering
\caption{Asymmetry of the transfer map. $r(B,B^{\mathsf{T}})$ is $1$ for
symmetric transfer and is computed over the $56$ off-diagonal entries, that is, $28$ pairs in both orientations. The skew fraction,
$\lVert (B-B^{\mathsf{T}})/2\rVert_F/\lVert B\rVert_F$ over the off-diagonal
entries, equals $\sqrt{(1-r)/2}$ here, so it restates $r$; independent noise
sits near $0.707$. \emph{Directional} counts pairs whose difference interval
excludes zero. The last three columns count pairs whose two directions have
opposite signs: as point estimates, with the difference interval also
excluding zero, and with each direction's own interval excluding zero (the
count Section~\ref{sec:asym} reports).}
\label{tab:asym}
\footnotesize\setlength{\tabcolsep}{4pt}
\begin{tabular}{lrrrrrr}
\toprule
& & & & \multicolumn{3}{c}{Opposite signs} \\
\cmidrule(lr){5-7}
Condition & $r(B,B^{\mathsf{T}})$ & Skew & Directional & Estimates & Diff.\ CI & Both CIs \\
\midrule
Qwen3-32B & $+0.304$ $[+0.177,+0.409]$ & $0.590$ & $14/28$ & $10/28$ & $7/28$ & $3/28$ \\
Mistral-24B & $+0.290$ $[+0.167,+0.381]$ & $0.596$ & $13/28$ & $15/28$ & $10/28$ & $4/28$ \\
PubMed & $+0.262$ $[+0.052,+0.393]$ & $0.607$ & $9/28$ & $10/28$ & $3/28$ & $0/28$ \\
\bottomrule
\end{tabular}
\end{table}

On Qwen3-32B, three pairs meet the strictest count:
\task{DEF}/\task{MHOP}, \task{ERR}/\task{WHY} and \task{WHY}/\task{DEF}
(Table~\ref{tab:asympairs}). On Mistral-24B, four do: \task{WHY}/\task{MHOP},
\task{ERR}/\task{WHY}, \task{WHY}/\task{DEF} and \task{MATCH}/\task{SC-MC}; the
last has entries of only $-0.32$ and $+0.20\pp$ and involves the near-ceiling
\task{MATCH}. PubMed has none. Every opposite-sign pair that meets the difference-interval
count but not the strictest one has an entry inside $\pm0.5\pp$ or involves
\task{MATCH}. Table~\ref{tab:asympairs} lists
every pair for the Qwen3-32B condition, ordered by the size of the difference
between directions; the other two conditions are in
\texttt{results/asymmetry.md}. The largest asymmetry is the same pair in all
three conditions: \task{WHY}$\to$\task{MHOP} minus \task{MHOP}$\to$\task{WHY}
is $-8.16\pp$ $[-10.51,-5.71]$ on Qwen3-32B, $-7.09\pp$ $[-9.01,-5.20]$ on
Mistral-24B and $-4.81\pp$ $[-6.89,-2.52]$ on PubMed.

\begin{table}[h]
\centering
\caption{Qwen3-32B: every task pair $(A, B)$, both directions. \emph{First} is
$\beta_{A\to B}$ and \emph{second} is $\beta_{B\to A}$, as centered substitution
effects in pp; the difference is first minus second. All intervals are $95\%$
percentile bootstrap intervals over distinct mixtures. $^{*}$ marks the pairs
whose two directions have opposite signs with each direction's interval
excluding zero.}
\label{tab:asympairs}
\footnotesize\setlength{\tabcolsep}{3pt}
\begin{tabular}{lrlrlrl}
\toprule
Pair $(A,B)$ & First & 95\% CI & Second & 95\% CI & Difference & 95\% CI \\
\midrule
\task{WHY}, \task{MHOP} & $+2.55$ & $[+1.20,+4.02]$ & $+10.71$ & $[+8.65,+12.52]$ & $-8.16$ & $[-10.51,-5.71]$ \\
\task{MHOP}, \task{DEF}$^{*}$ & $-1.24$ & $[-1.68,-0.75]$ & $+5.27$ & $[+3.55,+6.84]$ & $-6.51$ & $[-8.17,-4.74]$ \\
\task{ERR}, \task{MHOP} & $-9.72$ & $[-11.29,-7.93]$ & $-3.85$ & $[-4.81,-2.89]$ & $-5.88$ & $[-7.77,-3.82]$ \\
\task{ERR}, \task{WHY}$^{*}$ & $-2.37$ & $[-4.18,-0.61]$ & $+3.11$ & $[+2.35,+3.90]$ & $-5.48$ & $[-7.59,-3.58]$ \\
\task{WHY}, \task{DEF}$^{*}$ & $-2.25$ & $[-2.69,-1.79]$ & $+2.49$ & $[+0.32,+4.48]$ & $-4.73$ & $[-6.77,-2.61]$ \\
\task{WHY}, \task{MC-MC} & $-0.29$ & $[-0.54,-0.06]$ & $-3.72$ & $[-5.56,-1.86]$ & $+3.43$ & $[+1.50,+5.29]$ \\
\task{WHY}, \task{MATCH} & $+0.05$ & $[-0.06,+0.16]$ & $-2.72$ & $[-4.62,-0.77]$ & $+2.76$ & $[+0.81,+4.72]$ \\
\task{WHY}, \task{SC-MC} & $-0.05$ & $[-0.17,+0.07]$ & $-2.53$ & $[-4.44,-0.69]$ & $+2.48$ & $[+0.63,+4.39]$ \\
\task{WHY}, \task{CLZ} & $-0.13$ & $[-0.31,+0.03]$ & $-1.86$ & $[-3.86,+0.42]$ & $+1.74$ & $[-0.53,+3.76]$ \\
\task{ERR}, \task{DEF} & $+0.69$ & $[+0.25,+1.12]$ & $+2.33$ & $[+1.32,+3.25]$ & $-1.64$ & $[-2.66,-0.55]$ \\
\task{MC-MC}, \task{DEF} & $+1.16$ & $[+0.76,+1.60]$ & $+0.29$ & $[-0.00,+0.60]$ & $+0.87$ & $[+0.35,+1.40]$ \\
\task{MHOP}, \task{MC-MC} & $+0.69$ & $[+0.41,+0.97]$ & $-0.09$ & $[-1.55,+1.37]$ & $+0.78$ & $[-0.72,+2.25]$ \\
\task{MHOP}, \task{CLZ} & $+0.50$ & $[+0.32,+0.65]$ & $+1.27$ & $[+0.00,+2.88]$ & $-0.78$ & $[-2.41,+0.52]$ \\
\task{MHOP}, \task{MATCH} & $+0.07$ & $[-0.07,+0.19]$ & $+0.75$ & $[-0.83,+2.36]$ & $-0.69$ & $[-2.29,+0.92]$ \\
\task{SC-MC}, \task{DEF} & $+0.66$ & $[+0.26,+1.03]$ & $+0.10$ & $[-0.03,+0.22]$ & $+0.56$ & $[+0.15,+0.95]$ \\
\task{MATCH}, \task{DEF} & $+0.35$ & $[-0.06,+0.72]$ & $-0.04$ & $[-0.17,+0.11]$ & $+0.39$ & $[-0.03,+0.78]$ \\
\task{MHOP}, \task{SC-MC} & $+0.29$ & $[+0.15,+0.43]$ & $-0.03$ & $[-1.50,+1.51]$ & $+0.32$ & $[-1.25,+1.78]$ \\
\task{ERR}, \task{CLZ} & $-0.27$ & $[-0.43,-0.11]$ & $-0.58$ & $[-1.44,+0.36]$ & $+0.30$ & $[-0.64,+1.16]$ \\
\task{MATCH}, \task{SC-MC} & $-0.26$ & $[-0.41,-0.13]$ & $+0.02$ & $[-0.10,+0.14]$ & $-0.28$ & $[-0.48,-0.10]$ \\
\task{MATCH}, \task{CLZ} & $-0.18$ & $[-0.32,-0.03]$ & $+0.07$ & $[-0.07,+0.20]$ & $-0.25$ & $[-0.44,-0.05]$ \\
\task{SC-MC}, \task{CLZ} & $-0.20$ & $[-0.34,-0.08]$ & $+0.03$ & $[-0.12,+0.20]$ & $-0.23$ & $[-0.46,-0.03]$ \\
\task{MATCH}, \task{MC-MC} & $+0.28$ & $[-0.01,+0.59]$ & $+0.12$ & $[+0.01,+0.24]$ & $+0.16$ & $[-0.16,+0.50]$ \\
\task{ERR}, \task{MATCH} & $-0.29$ & $[-0.40,-0.17]$ & $-0.14$ & $[-0.94,+0.72]$ & $-0.15$ & $[-1.01,+0.65]$ \\
\task{ERR}, \task{SC-MC} & $-0.07$ & $[-0.20,+0.06]$ & $-0.18$ & $[-0.99,+0.72]$ & $+0.11$ & $[-0.81,+0.90]$ \\
\task{MC-MC}, \task{CLZ} & $-0.26$ & $[-0.41,-0.11]$ & $-0.34$ & $[-0.68,+0.02]$ & $+0.08$ & $[-0.33,+0.46]$ \\
\task{CLZ}, \task{DEF} & $+0.62$ & $[+0.07,+1.13]$ & $+0.54$ & $[+0.36,+0.74]$ & $+0.08$ & $[-0.49,+0.62]$ \\
\task{ERR}, \task{MC-MC} & $-0.62$ & $[-0.92,-0.29]$ & $-0.69$ & $[-1.44,+0.10]$ & $+0.08$ & $[-0.79,+0.91]$ \\
\task{MC-MC}, \task{SC-MC} & $-0.03$ & $[-0.16,+0.10]$ & $-0.02$ & $[-0.41,+0.30]$ & $-0.01$ & $[-0.36,+0.40]$ \\
\bottomrule
\end{tabular}
\end{table}

\clearpage
\section{Selector Comparison, Full Sweep}
\label{app:selectors}

Table~\ref{tab:sweepfull} lists every competitor at every trained mixture size
$d$, under one fixed training seed.
Positive favours the map-selected mixture. \emph{other-6} is the macro over the
six non-target, non-\task{MATCH} tasks and reads as a retention cost.
$\dagger$ marks a run whose generations tripped the degeneration check of
Appendix~\ref{app:details}. PubMed/\task{ERR} is the vacuous cell, so its interferer column reads n/a
(Table~\ref{tab:selectors}).

Admitting an interferer is not the only way a selector loses. At its matched
size the Mistral/\task{MHOP} similarity mixture excludes \task{ERR} and
still trails the map by $9.68\pp$, because it omits \task{DEF}, the source
with the largest positive entry for that target. On the same cell, TaskShop at
$d=5$ keeps all three of the map's sources and still trails by $8.42\pp$ after
adding two that the map drops, \task{SC-MC} and \task{WHY}.
Skill-It edges the map on Qwen/\task{WHY}, Qwen/\task{ERR} and
Mistral/\task{WHY}, by at most $1.45\pp$, and loses $17.54\pp$ on
PubMed/\task{WHY}, where its mixture holds that target's interferer.

Selected mixtures are narrower than the uniform default. Over the
out-of-fold selections of Section~\ref{sec:scope} the map keeps $3.2$,
$3.0$ and $4.1$ sources on average in the three conditions, against seven;
over the nine predicted cells it keeps three to seven (Table~\ref{tab:cells}).
At the fixed $6{,}000$-example budget that is $1.2$ to $2.3\times$ as many
examples per source in the eight cells where the selection is narrower.

\begin{figure}[h]
  \centering
  \includegraphics[width=\textwidth]{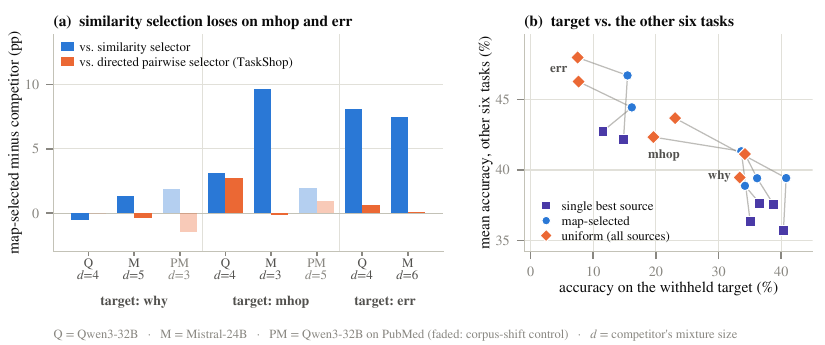}
  \caption{\textbf{The similarity selector loses where selection matters.}
  \textbf{(a)} Map-selected minus competitor accuracy on the withheld
  target, each competitor at its own mixture size $d$. \textbf{(b)} Accuracy
  on the target against the non-target macro (the mean over the six scored
  tasks other than the target and \task{MATCH}): the map-selected mixture
  sits between the uniform mixture and the single best source.}
  \label{fig:selectors}
\end{figure}

\begin{center}
\footnotesize
\begin{longtable}{lllrlrr}
\caption{Map minus competitor, on the withheld target and on the non-target
macro.}\label{tab:sweepfull}\\
\toprule
Cond. & Target & Competitor & $d$ & Interferer in mix & target & other-6 \\
\midrule
\endfirsthead
\toprule
Cond. & Target & Competitor & $d$ & Interferer in mix & target & other-6 \\
\midrule
\endhead
\bottomrule
\endfoot
Qwen & \task{WHY} & \texttt{uniform} & $7$ & yes (\task{MC-MC}) & $+1.21$ & $-0.75$ \\
Qwen & \task{WHY} & \texttt{single\_best} & $1$ & no & $-2.37$ & $+1.25$ \\
Qwen & \task{WHY} & \texttt{skillit} & $4$ & no & $-1.45$ & $+1.41$ \\
Qwen & \task{WHY} & \texttt{taskshop\_k3} & $3$ & no & $-1.36$ & $+0.04$ \\
Qwen & \task{WHY} & \texttt{taskshop} & $4$ & no & $+0.00$ & $+0.00$ \\
Qwen & \task{WHY} & \texttt{taskshop\_k5} & $5$ & no & $-2.13$ & $-0.53$ \\
Qwen & \task{WHY} & \texttt{embed\_sim\_k3} & $3$ & no & $-1.45$ & $-0.46$ \\
Qwen & \task{WHY} & \texttt{embed\_sim} & $4$ & no & $-0.53$ & $-0.01$ \\
Qwen & \task{WHY} & \texttt{embed\_sim\_k5} & $5$ & yes (\task{MC-MC}) & $-0.63$ & $-0.33$ \\
\addlinespace
Qwen & \task{MHOP} & \texttt{uniform} & $7$ & yes (\task{ERR}) & $+13.19$ & $-1.28$ \\
Qwen & \task{MHOP} & \texttt{single\_best} & $1$ & no & $-1.41$$^{\dagger}$ & $+4.97$ \\
Qwen & \task{MHOP} & \texttt{skillit} & $4$ & no & $+2.43$$^{\dagger}$ & $+0.78$ \\
Qwen & \task{MHOP} & \texttt{taskshop\_k3} & $3$ & no & $-1.07$ & $-0.01$ \\
Qwen & \task{MHOP} & \texttt{taskshop} & $4$ & no & $+2.73$ & $-0.22$ \\
Qwen & \task{MHOP} & \texttt{taskshop\_k5} & $5$ & no & $+1.85$ & $-0.21$ \\
Qwen & \task{MHOP} & \texttt{embed\_sim\_k3} & $3$ & no & $+3.41$ & $-0.04$ \\
Qwen & \task{MHOP} & \texttt{embed\_sim} & $4$ & no & $+3.11$ & $-0.04$ \\
Qwen & \task{MHOP} & \texttt{embed\_sim\_k5} & $5$ & yes (\task{ERR}) & $+12.90$ & $-0.45$ \\
\addlinespace
Qwen & \task{ERR} & \texttt{uniform} & $7$ & yes (\task{MHOP}) & $+8.64$ & $-1.50$ \\
Qwen & \task{ERR} & \texttt{single\_best} & $1$ & no & $+2.11$$^{\dagger}$ & $+4.58$ \\
Qwen & \task{ERR} & \texttt{skillit} & $4$ & no & $-0.20$ & $+0.25$ \\
Qwen & \task{ERR} & \texttt{taskshop\_k3} & $3$ & no & $+0.88$ & $-0.14$ \\
Qwen & \task{ERR} & \texttt{taskshop} & $4$ & no & $+0.64$ & $+0.33$ \\
Qwen & \task{ERR} & \texttt{taskshop\_k5} & $5$ & no & $+0.98$ & $+0.12$ \\
Qwen & \task{ERR} & \texttt{embed\_sim\_k3} & $3$ & yes (\task{MHOP}) & $+8.35$ & $-1.31$ \\
Qwen & \task{ERR} & \texttt{embed\_sim} & $4$ & yes (\task{MHOP}) & $+8.10$ & $-1.77$ \\
Qwen & \task{ERR} & \texttt{embed\_sim\_k5} & $5$ & yes (\task{MHOP}) & $+8.99$ & $-0.82$ \\
\addlinespace
Mistral & \task{WHY} & \texttt{uniform} & $7$ & yes (\task{ERR}) & $+2.18$ & $-1.50$ \\
Mistral & \task{WHY} & \texttt{single\_best} & $1$ & no & $-2.13$ & $+1.64$ \\
Mistral & \task{WHY} & \texttt{skillit} & $5$ & no & $-1.31$ & $+0.59$ \\
Mistral & \task{WHY} & \texttt{taskshop\_k3} & $3$ & no & $-0.63$ & $+1.14$ \\
Mistral & \task{WHY} & \texttt{taskshop} & $5$ & no & $-0.34$ & $-0.25$ \\
Mistral & \task{WHY} & \texttt{embed\_sim\_k3} & $3$ & yes (\task{ERR}) & $+1.45$ & $-0.82$ \\
Mistral & \task{WHY} & \texttt{embed\_sim} & $5$ & yes (\task{ERR}) & $+1.36$ & $-1.63$ \\
\addlinespace
Mistral & \task{MHOP} & \texttt{uniform} & $7$ & yes (\task{ERR}) & $+17.23$ & $-4.25$ \\
Mistral & \task{MHOP} & \texttt{single\_best} & $1$ & no & $+0.44$$^{\dagger}$ & $+3.62$ \\
Mistral & \task{MHOP} & \texttt{skillit} & $3$ & no & $+7.59$ & $+2.33$ \\
Mistral & \task{MHOP} & \texttt{taskshop} & $3$ & no & $-0.15$ & $+0.11$ \\
Mistral & \task{MHOP} & \texttt{taskshop\_k5} & $5$ & no & $+8.42$ & $-4.06$ \\
Mistral & \task{MHOP} & \texttt{embed\_sim} & $3$ & no & $+9.68$ & $-1.84$ \\
Mistral & \task{MHOP} & \texttt{embed\_sim\_k5} & $5$ & yes (\task{ERR}) & $+15.91$ & $-4.13$ \\
\addlinespace
Mistral & \task{ERR} & \texttt{uniform} & $7$ & yes (\task{MHOP}) & $+7.76$ & $-1.29$ \\
Mistral & \task{ERR} & \texttt{single\_best} & $1$ & no & $+1.42$ & $+4.98$ \\
Mistral & \task{ERR} & \texttt{skillit} & $6$ & yes (\task{MHOP}) & $+7.91$ & $-0.66$ \\
Mistral & \task{ERR} & \texttt{taskshop\_k3} & $3$ & no & $-0.34$ & $+0.99$ \\
Mistral & \task{ERR} & \texttt{taskshop} & $6$ & no & $+0.10$ & $+0.14$ \\
Mistral & \task{ERR} & \texttt{taskshop\_k5} & $5$ & no & $-1.23$ & $+0.61$ \\
Mistral & \task{ERR} & \texttt{embed\_sim\_k3} & $3$ & yes (\task{MHOP}) & $+8.25$ & $+0.48$ \\
Mistral & \task{ERR} & \texttt{embed\_sim} & $6$ & yes (\task{MHOP}) & $+7.47$ & $-1.38$ \\
Mistral & \task{ERR} & \texttt{embed\_sim\_k5} & $5$ & yes (\task{MHOP}) & $+7.86$ & $-0.44$ \\
\addlinespace
PubMed & \task{WHY} & \texttt{uniform} & $7$ & yes (\task{MC-MC}) & $+0.00$ & $-3.02$ \\
PubMed & \task{WHY} & \texttt{single\_best} & $1$ & no & $+4.27$ & $+0.46$ \\
PubMed & \task{WHY} & \texttt{skillit} & $3$ & yes (\task{MC-MC}) & $+17.54$ & $+3.00$ \\
PubMed & \task{WHY} & \texttt{taskshop} & $3$ & no & $-1.42$ & $-2.79$ \\
PubMed & \task{WHY} & \texttt{taskshop\_k5} & $5$ & yes (\task{MC-MC}) & $+2.84$ & $-2.83$ \\
PubMed & \task{WHY} & \texttt{embed\_sim} & $3$ & yes (\task{MC-MC}) & $+1.90$ & $-0.27$ \\
PubMed & \task{WHY} & \texttt{embed\_sim\_k5} & $5$ & yes (\task{MC-MC}) & $+4.74$ & $-1.24$ \\
\addlinespace
PubMed & \task{MHOP} & \texttt{uniform} & $7$ & yes (\task{ERR}) & $+2.45$ & $-5.58$ \\
PubMed & \task{MHOP} & \texttt{single\_best} & $1$ & no & $+4.90$ & $+0.13$ \\
PubMed & \task{MHOP} & \texttt{skillit} & $5$ & no & $+1.47$ & $-4.06$ \\
PubMed & \task{MHOP} & \texttt{taskshop\_k3} & $3$ & no & $-0.49$ & $-1.00$ \\
PubMed & \task{MHOP} & \texttt{taskshop} & $5$ & no & $+0.98$ & $+0.14$ \\
PubMed & \task{MHOP} & \texttt{embed\_sim\_k3} & $3$ & no & $+5.39$ & $-1.55$ \\
PubMed & \task{MHOP} & \texttt{embed\_sim} & $5$ & no & $+1.96$ & $-3.98$ \\
\addlinespace
PubMed & \task{ERR} & \texttt{uniform} & $7$ & n/a & $+1.69$ & $+0.24$ \\
PubMed & \task{ERR} & \texttt{single\_best} & $1$ & n/a & $+0.00$ & $+6.96$ \\
PubMed & \task{ERR} & \texttt{taskshop\_k3} & $3$ & n/a & $+1.69$ & $-0.49$ \\
PubMed & \task{ERR} & \texttt{taskshop\_k5} & $5$ & n/a & $+1.27$ & $+1.43$ \\
PubMed & \task{ERR} & \texttt{embed\_sim\_k3} & $3$ & n/a & $+2.11$ & $+2.15$ \\
PubMed & \task{ERR} & \texttt{embed\_sim\_k5} & $5$ & n/a & $+0.84$ & $+3.31$ \\
\end{longtable}
\end{center}

\section{Training-Order Experiment}
\label{app:order}

We test whether the map also predicts the effect of training order. Each pair
trains two sources in two consecutive blocks, once in each order, and we score it
on one target. The recorded prediction is that target accuracy is higher when the
source with the larger map effect on the target (the \emph{better} source) comes
last, by half the transfer differential, the difference between the two sources'
map effects on the target.
Table~\ref{tab:order} compares every pair's realized effect with this
prediction. We trained $99$ runs: $30$ in each block order, $30$ with the two
sources shuffled together (the default), and
$9$
runs of \texttt{P1} under a constant learning rate. Pairs \texttt{P1} to \texttt{P8} train Qwen3-4B,
and \texttt{P1r17} and \texttt{P1r8} repeat \texttt{P1} at Qwen3-1.7B and
Qwen3-8B, each scored against the map estimated at its own size. Within a pair the
two blocked runs train on a bit-identical multiset of examples (same rows,
same per-task budget, same steps, same seed), so training order is the only
difference between them. The realized effect is
the seed-paired target accuracy of the run that trains the better source \emph{last}
minus that of the run that trains it first, so the predicted sign is positive. Three
seeds per order put the minimum detectable effect at about
$1.97\pp$, and we call an interval that covers zero
unresolved and never read it as evidence of no effect. We recorded three pairs,
\texttt{P3}, \texttt{P6} and \texttt{P8}, as placebos before we measured their
transfer differentials; their predicted values apply the recorded rule to differentials measured
afterwards.

\begin{table}[h]
\centering
\caption{Order effect on the target for every pair, in percentage points.
\emph{Realized} is the seed-paired accuracy of the run that trains the better
source last minus that of the run that trains it first. \emph{Predicted} is half
the transfer differential, the difference between the two sources' map effects on
the target; for the placebos it was computed after the predictions were recorded. \emph{Confirmed}
means the interval excludes zero in the predicted direction, and \emph{reversed}
that it excludes zero against it. \emph{Last vs shuffled} is the run that trains
the better source last minus the run that trains the same examples shuffled
together, over the same three seeds. The verdict in the text counts the
seven rows of the first two blocks.}
\label{tab:order}
\scriptsize\setlength{\tabcolsep}{2.5pt}
\begin{tabular}{lrllrlll}
\toprule
Pair & Size & Sources & Target & Predicted & Realized [95\% CI] & Last vs shuffled [95\% CI] & Status \\
\midrule
\multicolumn{8}{l}{\emph{\texttt{P1} at three model sizes}} \\
\midrule
\texttt{P1} & 4B & \task{DEF} + \task{ERR} & \task{MHOP} & $+8.32$ & $\mathbf{+9.55}$ $[+6.89,+12.21]$ & $+7.75$ $[+4.48,+11.02]$ & confirmed \\
\texttt{P1r17} & 1.7B & \task{DEF} + \task{ERR} & \task{MHOP} & $+7.71$ & $\mathbf{+8.21}$ $[+7.11,+9.30]$ & $+4.72$ $[+2.91,+6.53]$ & confirmed \\
\texttt{P1r8} & 8B & \task{DEF} + \task{ERR} & \task{MHOP} & $+7.24$ & $\mathbf{+9.62}$ $[+2.17,+17.07]$ & $+7.04$ $[+3.36,+10.72]$ & confirmed \\
\midrule
\multicolumn{8}{l}{\emph{Other recorded pairs}} \\
\midrule
\texttt{P2} & 4B & \task{MHOP} + \task{SC-MC} & \task{WHY} & $+4.57$ & $\mathbf{-0.66}$ $[-1.74,+0.42]$ & $+0.97$ $[-1.72,+3.66]$ & covers $0$ \\
\texttt{P4} & 4B & \task{DEF} + \task{MC-MC} & \task{WHY} & $+4.55$ & $\mathbf{-5.51}$ $[-8.24,-2.77]$ & $-3.84$ $[-7.60,-0.09]$ & reversed \\
\texttt{P5} & 4B & \task{MHOP} + \task{MC-MC} & \task{WHY} & $+5.28$ & $\mathbf{-0.27}$ $[-3.01,+2.46]$ & $+1.28$ $[+0.93,+1.62]$ & covers $0$ \\
\texttt{P7} & 4B & \task{SC-MC} + \task{WHY} & \task{DEF} & $+3.47$ & $\mathbf{+2.08}$ $[-0.87,+5.03]$ & $-0.69$ $[-3.29,+1.90]$ & covers $0$ \\
\midrule
\multicolumn{8}{l}{\emph{Recorded as placebos}} \\
\midrule
\texttt{P3} & 4B & \task{CLZ} + \task{MC-MC} & \task{ERR} & $+0.01$ & $\mathbf{+0.18}$ $[-0.38,+0.74]$ & $-0.49$ $[-1.02,+0.04]$ & covers $0$ \\
\texttt{P6} & 4B & \task{MHOP} + \task{CLZ} & \task{WHY} & $+4.20$ & $\mathbf{-0.05}$ $[-1.82,+1.72]$ & $+1.16$ $[+0.01,+2.31]$ & covers $0$ \\
\texttt{P8} & 4B & \task{DEF} + \task{CLZ} & \task{WHY} & $+3.47$ & $\mathbf{-4.10}$ $[-7.84,-0.37]$ & $-3.68$ $[-7.02,-0.34]$ & reversed \\
\midrule
\multicolumn{8}{l}{\emph{Constant learning rate}} \\
\midrule
\texttt{P1} & 4B & \task{DEF} + \task{ERR} & \task{MHOP} & $+8.32$ & $\mathbf{+15.12}$ $[+10.68,+19.55]$ & $+12.15$ $[+5.54,+18.76]$ & confirmed \\
\bottomrule
\end{tabular}
\end{table}

Over the $7$
rows of the first two blocks the verdict is $3$
confirmed, $3$ covering zero and $1$
reversed. The three confirmed rows are one source pair at three sizes, so the
prediction holds for one of five source pairs. For \texttt{P1}, training \task{DEF}
last gains $+7.75\pp$ on \task{MHOP} over the shuffled default at the same training
cost; the last column of Table~\ref{tab:order} gives this contrast for every pair. Only
\texttt{P3} has a differential near zero, so it alone is a placebo in substance,
and its narrow interval covers zero. \texttt{P6} and \texttt{P8} have differentials
of $+8.40$ and
$+6.95\pp$, so they are placebos in name only; read as tests,
\texttt{P6} covers zero and \texttt{P8} reverses. The data refute the recorded
magnitude law, under which the effect is half the differential: the
intervals of \texttt{P2}, \texttt{P4}, \texttt{P5}, \texttt{P6} and \texttt{P8}
exclude their predicted values. The map therefore does not, in general, predict the
effect of training order. Under a constant learning rate the
\texttt{P1} effect grows from $+9.55$ to $+15.12\pp$. A decaying learning rate gives
the last block smaller updates and damps recency, so the effect is not an artifact of
the cosine schedule. The experiment uses one corpus, two-source blocks and three
seeds, and only \texttt{P1} was run at more than one model size or learning-rate
schedule. Two blocks are the strongest form of the
manipulation, and we expect finer interleaving to show smaller effects. Of the $36$
\task{DEF}$+$\task{ERR} runs, $25$ trip the \task{CLZ} generation
guardrail because, after two verbose sources, the model overruns the $1.4$-word
answers of the cloze task. \task{CLZ} is not the target of these runs and the
analysis reads only target accuracy, so we keep these runs and report them here.
\section{Experimental Details}
\label{app:details}

\paragraph{Evaluation inventory.}
Open tasks use the rule of Equation~\ref{eq:metric}; closed tasks use exact
structured matching.

\begin{table}[h]
\centering
\caption{Neuroscience evaluation inventory ($14{,}618$ items). The PubMed
condition uses the same task set over $1{,}464$ items.}
\label{tab:inventory}
\small
\begin{tabular}{llrrl}
\toprule
Task & Capability & Items & Ref.\ words & Scoring \\
\midrule
\task{DEF}   & definition with ontology fields & $2{,}166$ & $32.2$ & Eq.~\ref{eq:metric} \\
\task{CLZ}   & fill-in-the-blank span          & $1{,}265$ & $1.4$ & normalized exact match \\
\task{SC-MC} & one correct option              & $2{,}148$ & $7.8$ & option match \\
\task{MC-MC} & multiple correct options        & $2{,}032$ & $25.2$ & option-set match \\
\task{MATCH} & term--definition pairing        & $852$     & $7.9$ & pair-set match \\
\task{MHOP}  & two-evidence reasoning          & $2{,}055$ & $34.9$ & Eq.~\ref{eq:metric} \\
\task{WHY}   & causal/mechanistic explanation  & $2{,}064$ & $18.3$ & Eq.~\ref{eq:metric} \\
\task{ERR}   & methodological-validity critique& $2{,}036$ & $36.8$ & Eq.~\ref{eq:metric} \\
\bottomrule
\end{tabular}
\end{table}

\paragraph{The map does not track supervised token volume.}
Mean reference-answer length runs from $1.4$ words for \task{CLZ} to $36.8$ for
\task{ERR} (Table~\ref{tab:inventory}). Every run splits $6{,}000$ examples
equally and training masks the prompt, so a source's share of the supervised
tokens is proportional to that length, and a volume artifact would rank the
sources by it. The Spearman correlation between length and the centered
effects is $+0.010$, $-0.105$ and $+0.257$ by condition, largest on the
underpowered PubMed condition, and $+0.047$ over all $168$ entries. The two
sources with the largest mean effects, \task{DEF} and \task{MHOP}, sit at
$32.2$ and $34.9$ words, and the most negative, \task{ERR}, sits at $36.8$:
length orders none of them. This bounds supervised volume as an explanation;
passages remain unmatched (Section~\ref{sec:limitations}).

\paragraph{Probe designs.}
The four designs withhold $\{\task{MC-MC},\task{WHY}\}$,
$\{\task{MHOP},\task{ERR}\}$, $\{\task{DEF},\task{CLZ}\}$ and
$\{\task{SC-MC},\task{MATCH}\}$. Within a design, six cyclic windows over the
exposed sources balance source presence at each mixture size $d=1,\dots,5$, and
the $d=6$ pool trains once. Each condition trains $124$ run directories
covering $94$ distinct mixtures. The other $30$ directories repeat, at the same
seed, a mixture that another design also trains; on Mistral-24B they reproduce
every accuracy exactly, and on Qwen3-32B they reproduce $111$ of the $168$
duplicated mixture--task accuracies ($127$ on PubMed), the rest differing by at
most $0.79\pp$ ($3.27\pp$ on PubMed). Each reasoning target is absent in $76$ of
those directories, spanning $53$ distinct mixtures per condition. At the four
smaller sizes, one fixed $24$-run probe replaces the designs: two blocks withhold
$\{\task{MC-MC},\task{WHY}\}$ and $\{\task{MHOP},\task{ERR}\}$, and each trains
three cyclic windows over its six sources at each size $d=1,\dots,4$. The same
$20$ distinct mixtures train at every size, and each reasoning target's fit there
reads $18$ or $20$ runs. Because the generator samples task tags per passage, the
passages behind different task types overlap only partially: residual
passage-set Jaccard runs $0.51$--$0.53$.

\paragraph{Training.}
Every Qwen3 model is the post-trained release (Qwen3-0.6B to Qwen3-32B), and
every run trains and decodes with thinking disabled; Mistral-24B is
Mistral-Small-3.2-24B-Instruct-2506. All runs use response-masked causal language modeling with rsLoRA rank $16$,
$\alpha=16$, dropout $0.05$, and adapters on attention and multi-layer
perceptron (MLP) projections;
bf16, maximum sequence length $4{,}096$, one epoch, per-device batch size $8$,
gradient accumulation $2$, AdamW at $10^{-5}$, cosine decay, $3\%$ warm-up with a
$50$-step floor, weight decay $0.01$, gradient clipping at $1.0$. Each source in
a $d$-source run contributes $6000/d$ examples. Confirmation replicates vary
sampling and data order with the adapter initialization pinned; a separate
experiment also varies the adapter initialization and leaves every contrast
within noise.

\paragraph{Decoding and guardrails.}
Greedy decoding with task-specific token caps. A generation guardrail, the
degeneration check of the main text, flags a run when more than $5\%$ of answers
hit the cap, more than $10\%$ repeat fragments, or more than $20\%$ contain fewer
than three characters. It flags $11$ of the $124$ Qwen3-32B probe runs on the
neuroscience corpus, three of them on the \task{MHOP} column that the \task{MHOP}
fit reads; the Mistral-24B and PubMed probes have none. At the smaller sizes it
flags $18$, $8$, $9$ and $0$ of the $24$ probe runs at $0.6$B, $1.7$B, $4$B and
$8$B, mostly for cloze answers that hit the token cap; only one of these flags, a
\task{WHY} flag at $1.7$B, falls on a column that a reasoning-target fit reads. It
also flags $41$ runs in the order campaign. The reported fits retain all runs. Among the selector-comparison runs of Section~\ref{sec:buys}, four trip
guardrails: the single-best runs on Qwen/\task{MHOP}, Qwen/\task{ERR} and
Mistral/\task{MHOP}, and the Skill-It run on Qwen/\task{MHOP}; so do the
Qwen3-32B single-source confirmation runs for \task{MHOP} and \task{ERR}, both on
\task{CLZ}. Those specific runs serve as diagnostics and not as clean
comparisons.

\paragraph{Runs and cost.}
Counting each probe mixture of the three conditions once, the paper rests on $751$ fine-tuning runs:
$282$ probe mixtures in the three conditions, $96$ probe runs at the four smaller
sizes, $60$ confirmation runs at $32$B and $24$B (the $54$ of
Section~\ref{sec:predict} and the six single-source runs of
Figure~\ref{fig:selectors}b), $108$ confirmation runs at the smaller sizes, $70$
in the selector comparison, $99$ in the order experiment and $36$
adapter-initialization replicates. One run is one fine-tune of $6{,}000$ examples plus evaluation over
the full held-out split. The scheduler's accounting
(\texttt{data/runtimes/arm\_runtimes.json}) records $683$ completed runs and $540$
GPU-hours, including campaigns this paper does not report; it does not cover the
three probes or the order experiment. Its per-campaign mean wall clock runs
$0.47$--$1.22$ GPU-hours on a single H100 or H200; we exclude failed tasks,
because their wall clock measures a crash and not a run. At the Qwen3-32B rate of
$0.91$ GPU-hours per run, a probe of $94$ distinct mixtures costs roughly $85$
GPU-hours and a $24$-run probe roughly $22$; at the smaller sizes a $24$-run probe costs $11$ to
$16$. Across $60$ random draws of $24$ probe runs, the out-of-fold gain of
Section~\ref{sec:scope} on Qwen3-32B has a standard deviation of $0.33\pp$,
against $1.76\pp$ at $8$ runs.

\paragraph{Corpus hygiene.}
\task{MATCH} is one template repeated across its $852$ evaluation items and
behaves as a ceiling artifact, which is why Section~\ref{sec:scope} measures
transport over the $18$ non-\task{MATCH} entries. Of the unique validation
\task{DEF} questions, $11.4\%$ appear verbatim in training. No run we score for
\task{DEF} trains on \task{DEF}, so that overlap cannot reach a withheld target;
it does bound what the \task{DEF} column means as a source diagnostic. On the
neuroscience corpus the \task{CLZ} and \task{MATCH} training pools hold fewer than
$6{,}000$ examples, so their single-source runs repeat examples; the other six
pools are larger. The PubMed pools are much smaller, so every PubMed run repeats
examples of at least one source. The generator also repeats
some questions within a paper, in both the training and the evaluation split;
evaluation scores each repeat as a separate item, and the item counts above
include them.

\paragraph{Reproduction.}
\texttt{make analysis} and \texttt{make figures} recompute every table and figure
in this paper from a frozen evidence bundle; \texttt{make verify} and
\texttt{make verify-figures} confirm that a rerun reproduces the committed
artifacts exactly.

\section{The Scoring Rule, Audited}
\label{app:metric}

Recomputing every effect from retained per-item outputs reproduces all $54$
reported accuracies to within $7\times10^{-15}\pp$, so the reported effects are
the scoring rule applied to the stored generations and not a re-run.

The threshold-free variant is the mean of
$\max(s_{\mathrm{full}}, s_{\mathrm{avg}})$ with no threshold; under it all $8$ of
$8$ realized effects are positive, against $6$ of $8$ under
Equation~\ref{eq:metric}. Both cells that change sign are \task{WHY} cells
whose underlying similarity moved as the map predicted without enough items
crossing $\tau$. The recorded predictions are accuracies, so magnitudes
compare only after a through-origin rescaling. We fit that rescaling post hoc on
the same eight effects; after it, the predictions still beat the
mixture-agnostic baseline, $2.88$ against $6.03\pp$. Rescoring the same
generations at $\tau=0.80$ or $\tau=0.90$, with the recorded predictions
unchanged, gives MAE $2.81$ and $2.12\pp$ against $6.82$ and $3.82\pp$ for the
mixture-agnostic baseline ($2.31$ against $6.20\pp$ at $\tau=0.85$), so the
calibration is not an artifact of the chosen threshold.

The rule also interacts with output format. When the answer and the gold are
both one sentence, $s_{\mathrm{full}}=s_{\mathrm{avg}}$ and the item gets one test
against $\tau$; otherwise the two scores differ and it gets two. The paired runs
shift in sentence count, most sharply on Mistral/\task{ERR}, where $0.2\%$ of map-selected answers
run to more than one sentence against $94.4\%$ under the uniform mixture. A
format artifact would make the effect track that shift. Over the eight cells,
the correlation between the signed multi-sentence gap and the realized effect
is $-0.33$, too weak at $n=8$ to show or rule out an artifact. The largest effect, Mistral/\task{MHOP} at $+13.76\pp$, occurs
where the map-selected run is the more multi-sentence of the pair. Restricted
to the items a run answered in one sentence, which holds the answer's format
fixed, six of the eight effects keep their sign, and the two largest survive nearly
intact: $+11.44$ and $+13.22\pp$ against $+11.89$ and $+13.76\pp$.
Mistral/\task{ERR} falls from $+8.56$ to $-0.10\pp$ and PubMed/\task{MHOP} from
$+3.10$ to $-2.96\pp$; in both cells one run answers fewer than $6\%$ of items
in one sentence. The subsample conditions on a variable the treatment itself
moves, so it corroborates the effects without controlling for the shift.

Under the $s_{\mathrm{avg}}$ branch alone, seven of the eight effects keep their
sign: the two \task{ERR} effects grow, the Mistral/\task{MHOP} effect shrinks
from $+13.76$ to $+4.79\pp$, and PubMed/\task{MHOP} reverses. We do not treat that cell as established.

\end{document}